\documentclass[9pt,twocolumn,twoside]{pnas-new}

\usepackage{booktabs}
\usepackage{pifont}%
\usepackage{rotating}
\usepackage{amsfonts}
\usepackage{multirow}
\usepackage{color}
\usepackage{graphicx}
\usepackage{hyperref}
\usepackage{amsmath}
\usepackage{amsthm}
\newtheorem{theorem}{Theorem}
\newtheorem{lemma}{Lemma}
\newtheorem{definition}{Definition}
\usepackage{algorithm} %
\usepackage{flushend}
\usepackage{subcaption}
\usepackage{wrapfig}
\usepackage{breakurl}
\usepackage{enumitem}
\usepackage{xspace}
\usepackage[normalem]{ulem}
\usepackage{listings}
\usepackage{flushend}
\usepackage{footmisc}

\makeatletter
\newcommand{\@chapapp}{\relax}%
\makeatother
\usepackage[title,toc,titletoc,header]{appendix}

\usepackage{tikz}
\usetikzlibrary{positioning, shapes}

\usepackage[version=4]{mhchem} %
\usepackage{float} %

\usepackage{lipsum}

\newcommand{\DefMacro}[2]{\expandafter\newcommand\csname rmk-#1\endcsname{#2}}
\newcommand{\UseMacro}[1]{\csname rmk-#1\endcsname}

\newcommand{\XComment}[1]{}
\newcommand{\Space}[1]{}

\definecolor{gray}{RGB}{211,211,211}
\newcommand{\jbasicstyle}{\small\sffamily}

\newcommand{\jnumberstyle}{\scriptsize}

\lstdefinelanguage{pseudo}
{ morekeywords={for, in, break, continue, try, except, not,
  if,else,return,map,fieldElement_array_array40,fieldElement_array40},
  keywordstyle=\bfseries, lineskip=-0.1em, numbers=left,
  numberstyle=\jnumberstyle, numbersep=4pt, basicstyle=\jbasicstyle,
  breaklines=true, breakautoindent=true, tabsize=2,
  columns=fullflexible, morecomment=*[l][\textsl]{//},
  mathescape=true, }
\def\formulaname{Formula}

\newfloat{formula}{pb}{form}
\floatname{formula}{\formulaname}

\newcommand{\CheLU}{CheLU\xspace}
\newcommand{\ReLU}{ReLU\xspace}
\newcommand{\RReLU}{RReLU\xspace}
\newcommand{\BReLU}{BReLU\xspace}

\newcommand{\Section}[1]{\section{#1}}
\newcommand{\Subsection}[1]{\subsection{#1}}

\usepackage[disable,textwidth=40pt,textsize=scriptsize]{todonotes} %
\LetLtxMacro{\todom}{\todo}
\renewcommand{\todo}[1]{\textcolor{blue}{[\textbf{#1}]}}
\newcommand{\todoi}[1]{\todom[inline]{#1}}

\renewcommand\vec{\mathbf}

\def\longrightharpoonup{\relbar\joinrel\rightharpoonup}
\def\longleftharpoondown{\leftharpoondown\joinrel\relbar}
\def\longrightleftharpoons{\mathop{\vcenter{\hbox{\ooalign{\raise1pt\hbox{$\longrightharpoonup\joinrel$}\crcr\lower1pt\hbox{$\longleftharpoondown\joinrel$}}}}}}
\def\rxn{\mathop{\longrightarrow}\limits}   %
\def\revrxn{\mathop{\longrightleftharpoons}\limits}

\templatetype{pnasresearcharticle} %

\title{Programming and Training Rate-Independent Chemical Reaction Networks}

\author[a,1,2]{Marko Vasic}
\author[a,1,2]{Cameron Chalk}
\author[a]{Austin Luchsinger}
\author[a]{Sarfraz Khurshid}
\author[a,2]{David Soloveichik}

\affil[a]{The University of Texas at Austin, USA}

\leadauthor{Vasic, Chalk}

\significancestatement{
To program complex behavior in environments incompatible with electronic controllers such as within bioreactors or engineered cells, we need to turn to chemical information processors.
While chemical reactions can perform computation in the stoichiometric exchange of reactants for products (how many molecules of which reactants result in how many molecules of which products),
the control of reaction rates is usually thought to allow more complex computation.
Motivated by the fact that correct stoichiometry is easier to ensure than reaction rates, we provide a method for programming and training chemical computation by stoichiometry.
We show such computation can be straightforwardly programmed in a manner analogous to imperative programming,
and demonstrate the execution of neural networks capable of complex machine learning tasks.
}
  
\authordeclaration{A preliminary conference version of this work appeared as ref.~\cite{vasic2020deep}.}
\equalauthors{\textsuperscript{1}M.V. contributed equally to this work with C.C.}
\correspondingauthor{\textsuperscript{2}To whom correspondence should be addressed. E-mail: vasic@utexas.edu, ctchalk@utexas.edu, david.soloveichik@utexas.edu}


\begin{abstract}

Embedding computation in biochemical environments incompatible with traditional electronics is expected to have wide-ranging impact in synthetic biology, medicine, nanofabrication and other fields.
Natural biochemical systems are typically modeled by chemical reaction networks (CRNs), and  CRNs can be used as a specification language for synthetic chemical computation.
In this paper, we identify a class of CRNs called non-competitive (NC) whose equilibria are absolutely robust to reaction rates and kinetic rate law, because their behavior is captured solely by their stoichiometric structure.
Unlike prior work on rate-independent CRNs, checking non-competition and using it as a design criterion is easy and promises robust output.
We also present a technique to program NC-CRNs using well-founded deep learning methods, showing a translation procedure from rectified linear unit (ReLU) neural networks to NC-CRNs.
In the case of binary weight ReLU networks, our translation procedure is surprisingly tight in the sense that a single bimolecular reaction corresponds to a single ReLU node and vice versa. 
This compactness argues that neural networks may be a fitting paradigm for programming rate-independent chemical computation.
As proof of principle, we demonstrate our scheme with
numerical simulations of CRNs translated from neural networks trained on traditional machine learning datasets (IRIS and MNIST),
as well as tasks better aligned with potential biological applications including virus detection and spatial pattern formation.
\end{abstract}
 
\dates{This manuscript was compiled on \today}
\doi{\url{www.pnas.org/cgi/doi/10.1073/pnas.XXXXXXXXXX}}

\begin{document}

\maketitle
\thispagestyle{firststyle}
\ifthenelse{\boolean{shortarticle}}{\ifthenelse{\boolean{singlecolumn}}{\abscontentformatted}{\abscontent}}{}

\todom{CC: Reinsert author contributions and author declaration}

\dropcap{C}ompared to our remarkable capacity to build complex electronic circuits, we  lack in our ability to engineer sophisticated reaction networks like the regulatory networks prevalent in biology. 
Molecular programming aims to engineer synthetic chemical information processors of increasing complexity from first principles.
This approach yields control modules compatible with the chemical environments within natural or synthetic cells, bioreactors, and in-the-field diagnostics.
Such computation could, for example, recognize disease state based on chemical inputs and actuate drug delivery to the affected cell.

A key object of molecular programming are chemical reaction networks (CRNs).
CRNs formally model chemical concentrations changing due to coupled chemical reactions in a well-mixed solution.
Biological CRNs are often hard to analyze because, in general, they require working with systems of coupled non-linear differential equations capable of highly complex dynamical systems behavior such as multi-stability, oscillation and chaos~\cite{epstein1998introduction}.
However, in engineering we may aim at specific classes of CRNs that are easier to reason about.
One such class has recently emerged in which information processing occurs solely due to the stoichiometric exchange of the reactants for products rather than the reaction rate~\cite{chen2014rate}.
An example of such computation is the single irreversible reaction $A + B \to C$ which computes the minimum function in the sense that the concentration of $C$ converges to the minimum of the initial concentrations of $A$ and $B$.
By coupling multiple reactions, 
more complex functions can be computed. 
Although stoichiometric computation is limited to continuous piecewise linear functions (with possible discontinuities at the axes),
this class of functions is computationally powerful as evidenced by the ability to approximate arbitrary functions,
as well as the widespread use of continuous piecewise linear functions in machine learning (e.g., neural networks with the ReLU activation function, see below).
\todom{CC: Should this be "piecewise affine functions"?}

Besides ease of analysis, such stoichiometrically computing CRNs are absolutely robust to variations in kinetics (\emph{rate-independence}).
Computation carried out by stoichiometry alone is correct whether the system obeys standard mass-action kinetics, Hill-function or Michaelis-Menten kinetics, or any other kinetic laws,
and does not err if the system is not well-mixed.
Engineering may also be aided by the fact that, unlike factors contributing to reaction rates, the stoichiometry of reactants and products is inherently digital and can be set exactly by the nature of the reaction.
For example, if realized with DNA strand displacement cascades, the identity and stoichiometry of  reactants and products can be programmed by synthesizing DNA strands with specific parts that are identical or complementary~\cite{soloveichik2010dna,chen2013programmable,srinivas2017enzyme}.
Note that such reactions can be made effectively irreversible as they are strongly driven by the formation of new base pairs.\footnote{Although we are motivated mostly by engineering concerns,
some biological CRNs may exhibit similar stoichiometric, rate-independent behaviour as identified in searches of the Biomodels repository~\cite{degrand2020graphical}.
}

In the first part of the paper we develop a new technique for proving that a class of CRNs stoichiometrically computes the desired function.
We identify the \emph{non-competitive} property, which means that a species is consumed in at most one reaction (see later for a formal definition).
We show that for non-competitive CRNs, rate-independence can be verified and the function computed can be determined by simple reasoning analogous to sequential programming:
Although all reactions occur simultaneously with continuously varying rates, we can imagine, counter-factually, that reactions happen sequentially in a series of straight line segments.
Non-competition is easy to check, and further fully captures the computational power of stoichiometric computation.
Thus, non-competitive CRNs are a powerful class of CRNs for rationally programming chemical behavior.
All subsequent constructions in this paper are non-competitive, and their correctness is proven via the above technique.
\todom{Add:
If a non-competive CRN is feed-forward in the sense that XXX, 
the sequence of reactions to consider follows from the feed-forward order.}

In the second part of this paper,
motivated by the widespread use of neural networks to generate behavior that is not easily specified programmatically,
we show a natural way to specify rate-independent chemical input-output behavior through training.
Specifically, we show how (feed-forward) ReLU (Rectified Linear Unit) neural networks can be directly implemented by non-competitive CRNs. 
ReLU neural networks are one of the most successful types of neural networks for deep learning, prevalent in all areas of machine learning. 
Thus we provide a powerful paradigm for creating chemical systems with complex computational functionality not easily obtained by other means.

The key elements of our general (rational-weight) ReLU neural network implementation are the ReLU and the weight multiplication modules.
Our ReLU module consists of a single unimolecular and a single bimolecular reaction.
Our weight multiplication module uses a number of uni- and bimolecular reactions that is proportional to the number of bits of precision in the weight.
(Although weight multiplication can be performed with two high-order reactions, such reactions cannot easily be implemented and are slow.)

To simplify the construction even further we consider restricting the class of ReLU neural networks to have $\{-1,0,1\}$ weights.
Despite the restriction on the values of the weights, 
such \emph{binary-weight} ReLU neural networks are known to be powerful in solving machine learning tasks and are well-researched in deep learning community~\cite{courbariaux2015binaryconnect}.
Applying an optimized version of our construction to binary weight ReLU networks yields a surprisingly compact CRN with only a single bimolecular reaction per ReLU node (plus additional unimolecular reactions at the input layer).

Showing how two models of computing can simulate each other elucidates the computational power of one model in terms of the other. 
In the case of stoichiometrically computing CRNs and ReLU neural networks, they are both capable of computing arbitrary continuous piecewise linear functions.\todom{CC: piecewise affine?}
However, since the size of the CRN depends on the digits of precision of the weights, making a quantitative connection between the computational power of the two models (e.g., comparing the number of reactions versus number of ReLU nodes to achieve the same functionality) is difficult.
Nonetheless, in the case of binary weight ReLU networks, 
we can make a tight connection between binary weight ReLU and the subclass of non-competitive CRNs in which a reaction involves any species at most once and with unit stoichiometry.
We show that such \emph{CheLU} CRNs and binary-weight ReLU networks can be considered to be equivalent models of computing as they can simulate each other with the number of ReLU nodes equalling the number of bimolecular reactions.

In the last part of the paper, we demonstrate through examples our procedure of using binary-weight ReLU neural networks to embed functionality in CRNs.  
For each example, we train the neural network classifier, generate the resulting CRN, and numerically simulate the CRN under the usual mass-action kinetics.
The kinetic simulation confirms convergence to the expected output and provides additional information about convergence time.
First, we train classifiers
on the widely used machine learning datasets IRIS and
MNIST. 
Next, motivated by the envisioned application of molecular computation in medical diagnostics, 
we differentiate
between four viral infections using chemical information
as input (gene expression levels).
Finally, an important direction of chemical computation in synthetic biology lies in spatial pattern formation with applications in tissue and organ engineering~\cite{santos2019using}.
As an example of spatial pattern formation, we use a neural network to generate a 2D pattern (heart shape).

\begin{figure}[t]
  \centering
  \includegraphics[width=\columnwidth]{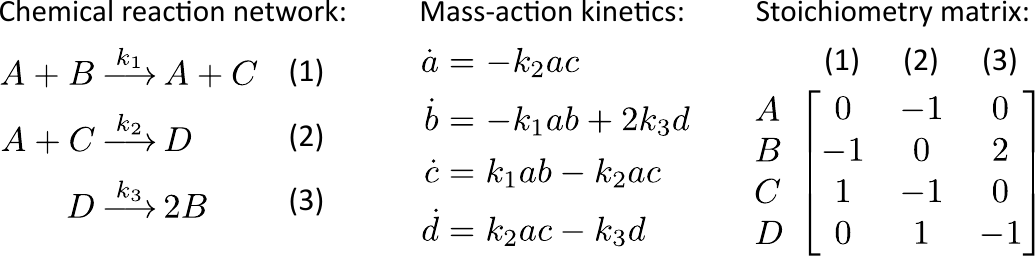}
  \caption{\textbf{Representations of chemical reaction networks.}
  The law of mass-action induces the differential equations describing the CRN's change in concentrations over time, where, e.g., $a$ represents the concentration of species $A$.
  The stoichiometry matrix captures the net change in species by each reaction, where entry $i,j$ corresponds to the change in species $i$ by applying reaction $j$.
}\label{fig:models}
\end{figure}
\section{Chemical Reaction Networks}\label{sec:CRNs}
Chemical reaction networks (CRNs) formally model the time evolution of molecules in a solution undergoing chemical interactions.
Besides the use of CRNs to capture the behavior of naturally existing chemical systems, synthetic biologists and molecular programmers often use CRNs as a programming language for rationally designed synthetic chemical networks such as DNA strand displacement cascades~\cite{chen2013programmable,srinivas2017enzyme} and DNA-enzyme networks~\cite{fujii2013predator}.
Related models of distributed computation include population protocols~\cite{angluin2006computation}, Petri nets~\cite{petri1966communication}, and vector addition systems~\cite{karp1969parallel}.

Next we provide some formal notation for CRNs aimed towards understanding the results of this work.
A CRN consists of a set of \emph{species} $\Lambda$ and a set of \emph{reactions}.
Reactions are written generally in this form:
$$r_1R_1 +  \dots + r_nR_n \rxn^k p_1P_1 + \dots + p_mP_m,$$
where $R_i, P_j \in \Lambda$ are the \emph{reactant} and \emph{product} species, respectively, the $r_i, p_j \in \mathbb{N}$ are \emph{stoichiometric coefficients} quantifying how much of each species is produced and how much is consumed, and $k$ is the rate constant used to describe the rate of the reaction in kinetic models like mass-action kinetics.
We note that although reactions written this way are irreversible, i.e., the products cannot react to form the reactants, in nature reactions always have some degree of reversibility.
However, synthetic chemical reactions can be made highly irreversible~\footnote{For example, implementing CRNs via DNA strand displacement  yields reactions which are driven by the formation of additional base pairs, and can be designed to be highly thermodynamically favorable~\cite{soloveichik2010dna,chen2013programmable,srinivas2017enzyme}.} and if desired this model can include the reverse of each reaction, e.g. $R_1 + R_2 \rxn P$ and $P \rxn R_1 + R_2$.
While the results of Section~\ref{sec:CRNprogramming} apply to reactions with arbitrarily many reactants,
the constructions in Sections~\ref{sec:rrelu} and~\ref{sec:brelu} consist of reactions with at most two reactants.
Reactions with more than two reactants are slow in practice, as they require the co-localization of more than two molecules before reactions can occur.
Further, while simulation of high-order reactions by bimolecular ones is possible, the typical method disturbs kinetics and does not fit in the non-competitive class (defined later) we are focusing on.\footnote{The typical method for simulating, e.g., the reaction $3X \rxn Y$ is to use the reactions $X + X \revrxn X_1$ and $X + X_1 \rxn Y$.}

A \emph{state} of a CRN is an assignment of nonnegative real-valued \emph{concentrations} (amount per volume) to each species. 
It helps to pick an arbitrary ordering on the species so that we can view states as vectors from $\mathbb{R}_{\geq 0}^\Lambda$ for compatibility with linear algebra techniques used later.
We use $\vec{a}(S)$ to denote the concentration of species $S$ in state $\vec{a}$.

CRNs are typically modeled either by differential equations or as stochastic processes.
Much of the discussion in this paper centers on the ubiquitous continuous mass-action kinetics model (example in Figure~\ref{fig:models}) which prescribes differential equations from reaction rates proportional to the product of the reactants' concentrations.
However, we focus on CRNs whose convergence state is independent of rate law, so assuming mass-action kinetics is not required for our theory to hold and constructed CRNs to compute correctly.
Further, an analogy of our Theorem~\ref{thm:main} holds for discrete stochastic models and is presented in SI Appendix~\ref{sec:stochastic_proofs}.

Next we present a \emph{nondeterministic kinetic model}, first proposed by~\cite{chen2014rate}, designed to isolate the effect of stoichiometry from the effect of rates.
This model does not intend to capture real-world chemical kinetics directly.
Instead, it is a simplified model that aids analysis of CRNs: as we will show, for the class of CRNs of interest, convergence in this simplified model implies convergence under mass-action kinetics and a wide variety of rate laws, even if the state of the CRN is initially perturbed.
Intuitively, the model explores the set of states reachable by the CRN assuming nothing about the kinetics besides that stoichiometry is obeyed.

The \emph{stoichiometry matrix} $\vec{M}$ captures the stoichiometric constraints of the CRN (example in Figure~\ref{fig:models}).
Assuming an ordering on species and reactions,
each column corresponds to a reaction, and each row to a species:
$\vec{M}_{ij}$ corresponds to the net increase/decrease of species $i$ by applying reaction $j$.

Recall that by arbitrarily ordering the set of species $\Lambda$, we can view states of the CRN as vectors of concentrations $\vec{a} \in \vec{R}^{\Lambda}_{\geq 0}$.
Then we can also describe \emph{flux vectors} which are column vectors $\vec{u} \in \vec{R}^{\Lambda}_{\geq 0}$ which describe arbitrary, simultaneous applications of reactions, which when multiplied by the stoichiometry matrix $\vec{M}$ yield the change in concentrations caused by applying those reactions.
Since $\vec{u}$ describes a set of reactions to happen, we say $\vec{u}$ is \emph{applicable} at a state $\vec{a}$ if all species which are reactants in the set of reactions in $\vec{u}$ have positive concentration in $\vec{a}$; formally, $\vec{u}$ is applicable at $\vec{a}$ if $\vec{u}(S) > 0$ implies that all reactants $R$ of reaction $S$ have $\vec{a}(R) > 0$.
For states $\vec{a}$ and $\vec{b}$, we say $\vec{a} \rightarrow^1_\vec{u} \vec{b}$ if there is a flux vector $\vec{u}$ applicable\footnote{Removing the applicability constraint would trivialize finding the set of reachable states of the CRN but would lead to erroneous analysis.
For example, given the CRN $X_1 + X_2 \rxn Y + Z$, $Z \rxn X_2$, given the ordering on species $X_1, X_2, Y, Z$ and an initial state $\vec{a} = [10, 0, 0, 0]$, state $\vec{b} = [0,0,10,0]$, and flux vector $\vec{u}= [10,10]$, we would have that $\vec{b} = \vec{M}\vec{u} + \vec{a}$, although from $\vec{a}$ no reactions should be applicable because there is initially zero concentration of $X_2$ and $Z$.} at $\vec{a}$ such that $\vec{b} = \vec{M}\vec{u} + \vec{a}$; this is \emph{straight-line} reachability.
Given $\vec{a} \rightarrow_{\vec{u}}^1 \vec{b}$, we say reaction $R$ is being \emph{applied} if $\vec{u}(R) > 0$.
We say $\vec{a} \rightarrow \vec{b}$ if there is a finite length sequence $\vec{a} \rightarrow^1 \dots \rightarrow^1 \vec{b}$, i.e., $\rightarrow$ is the transitive reflexive closure of $\rightarrow^1$; this is called \emph{line-segment} reachability.
If no flux vectors $\vec{u}$ besides the zero vector are applicable at state $\vec{b}$, then we call $\vec{b}$ a \emph{static} state.

\section{Programming CRN Computation by Stoichiometry}\label{sec:CRNprogramming}
The computational power of CRNs typically arises from both kinetics and stoichiometry. 
However, the equilibrium of certain CRNs can be understood entirely by the stoichiometric exchange of reactants for products (Figure~\ref{fig:divby2}).
Such systems have been used as an alternate paradigm for programming complex chemical behavior~\cite{chen2014deterministic,chen2014rate}, inspired by similar notions in distributed computing~\cite{angluin2006computation}.
We call such CRNs \emph{stoichiometrically defined}.\footnote{Previous work calls this notion \emph{stable computation}. We use the term \emph{stoichiometrically defined} to avoid confusion with other notions of stability in chemistry.}

\begin{figure}[]
  \centering
  \includegraphics[]{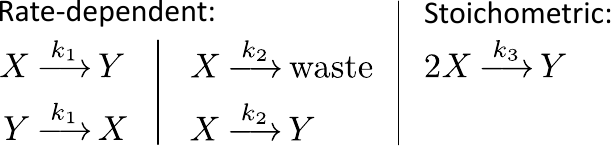}
  \caption{\textbf{Two rate-dependent CRNs and one stoichiometrically-defined CRN computing $y = \frac{x}{2}$.}
  The concentration of species $Y$ as time goes to infinity is half of the initial concentration of species $X$.
  The rate-dependent CRNs require that the rate constants of the two reactions are equal, while the third single-reaction CRN has no rate constraints.}\label{fig:divby2}
\end{figure}

\todom{Fig. 2. Rename "Stoichometric" to "Stoichiometric"}

To view CRNs as a method of computation (or, a programming language), we assign some species to be the inputs and others to be the outputs.
Then, given initial concentrations of the input species, the output of the computation is the equilibrium state of the system, i.e., the concentrations of the output species in the limit as time goes to infinity.~\footnote{There are alternative notions of computation by CRN; for example, a CRN may compute $f(t)$ in the sense that the concentration of a species is equal to $f(t)$ for all times $t$.}
Generally, given a function $f: \mathbb{R}^n_{\geq 0} \rightarrow \mathbb{R}^m_{\geq 0}$, some input species $X_1, \dots, X_n$ and an initial concentration assignment to each will represent an input vector $\vec{x}$, and output species $Y_1, \dots, Y_m$ and their respective concentrations at equilibrium will represent the output vector $\vec{y}$ such that $f(\vec{x}) = \vec{y}$.

A small example is the reaction $X_1 + X_2 \rxn Y$ which computes $f(x_1, x_2) = \min(x_1, x_2)$, since the reaction converges to a state where either $X_1$ or $X_2$, whichever has initially lower concentration, is depleted.
A more complex example computes $f(x_1, x_2) = \max(x_1, x_2)$ (Figure~\ref{fig:max}).

\begin{figure}[]
  \centering
  \includegraphics[]{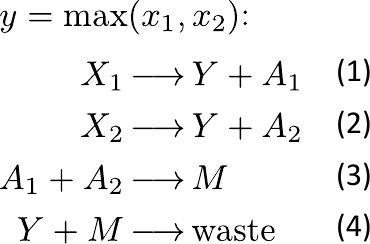}
  \hfill
  \includegraphics[]{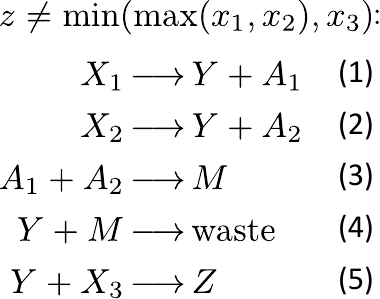}
  \caption{\textbf{(Left) A stoichiometrically-defined CRN computing the $\max$ function.}
Let $x_i(0)$ be the initial concentration of the input species $X_i$; all other initial concentrations are assumed to be $0$. Reactions (1) and (2) converge to an amount of $Y$ equal to $x_1(0) + x_2(0)$, and amounts of $A_1, A_2$ equal to $x_1(0), x_2(0)$, respectively.
Reaction (3) converges to an amount of $M$ equal to the $\min$ between the amounts of $A_1$ and $A_2$ produced by reactions (1) and (2), i.e., the $\min$ between $x_1(0)$ and $x_2(0)$.
In reaction (4), the $M$ species annihilate the $Y$ species, so that the concentration of $Y$ at convergence is decreased by the concentration of $M$, effectively computing subtraction.
In all, the amount of $Y$ converges to $x_1(0) + x_2(0) - \min(x_1(0), x_2(0)) = \max(x_1(0), x_2(0))$.
 Using Theorem~\ref{thm:main}, a formal argument of convergence is given by applying the reactions maximally, one-by-one, and in numerical order in the nondeterministic kinetic model.
 \textbf{(Right) Composing the $\max$ computing CRN with a $\min$ computing CRN does not yield a stoichiometrically-defined CRN computing $\min \circ \max$.}
  Reaction ($5$) attempts to use the output $Y$ of the $\max$ computing reactions shown in the left panel to compute $\min(y, x_3)$.
  This fails since reaction ($5$) might consume more $Y$ than $\max(x_1,x_2)$, and thus generate more than the correct amount of $Z$, by outcompeting reaction ($4$). 
  The extent of the error depends on the relative rates of reactions ($4$) and ($5$). 
  }\label{fig:max}
\end{figure}

\begin{figure}[]
  \centering
  \includegraphics[width=\columnwidth]{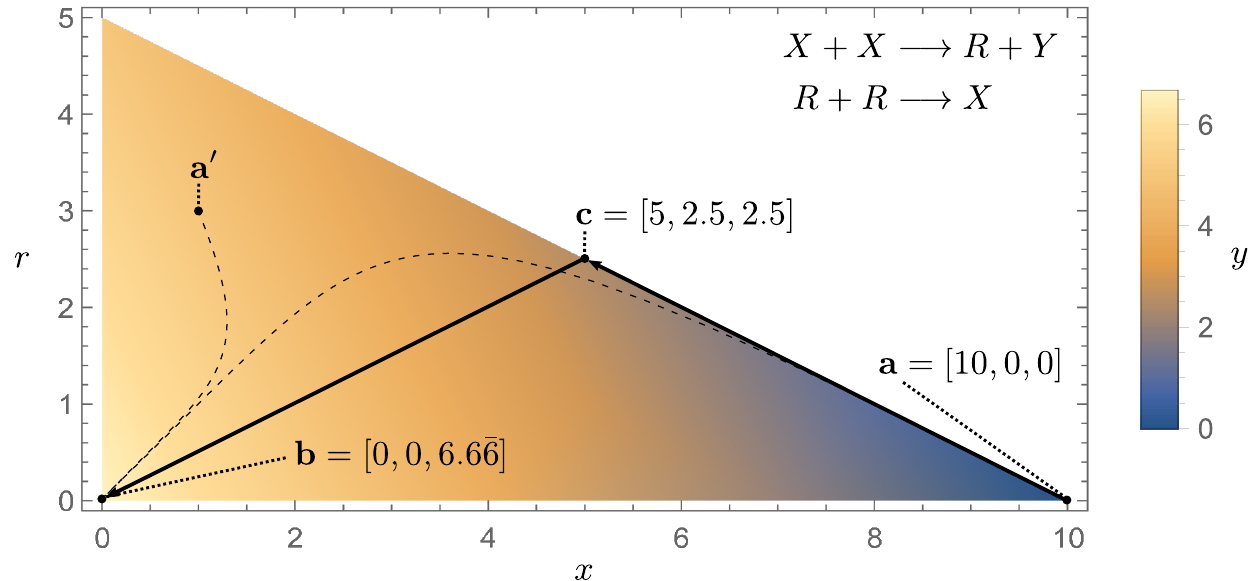}
  \caption{\textbf{Example application of Theorem~\ref{thm:main} on the non-competitive CRN $X + X \rxn R + Y$, $R + R \rxn X$ with initial state $\vec{a} = [10,0,0]$.}
  The shaded region shows all stoichiometrically reachable states from $\vec{a}$, i.e., all states $\vec{d}$ such that $\vec{a} \rightarrow \vec{d}$.
  Solid lines are straight-line reachable paths (specifically, $\vec{a}\rightarrow^1 \vec{c}$ and $\vec{c} \rightarrow^1 \vec{b}$) and dashed lines are mass-action trajectories (assuming both reactions have rate constant $1$, although the theorem applies to any rate constants).
  Since there is a path $\vec{a} \rightarrow \vec{b}$, Theorem~\ref{thm:main} implies that $\vec{a}$ also converges to $\vec{b}$ under mass action or any other fair rate law.
  Further, as shown by the state $\vec{a}'$, any state stoichiometrically reachable from $\vec{a}$ will also converge to $\vec{b}$ under mass action or any other fair rate law, showing that the convergence is robust to any initial perturbations that do not leave the stoichiometrically reachable space of states.
  }\label{fig:nc_theorem}
\end{figure}

\subsection{Non-competitive CRNs}
Here we identify a class of CRNs which we will show are easy to analyze and yet do not lose any computational power if we restrict to stoichiometrically defined, rate-independent computation.
To identify the class, note that an intuition for why the $\max$-computing CRN does not depend on rates is that each species is a reactant in at most one reaction, i.e., there is no competition between reactions for species.
For this reason, we find that reaction ($1$) of the $\max$-computing CRN must produce an amount of $Y$ and $A_1$ equal to the initial amount of $X_1$ as time goes to infinity, since $X_1$ cannot be decreased (nor increased) by any other reaction.
Reasoning about the other reactions similarly yields the correct output.
Carefully formalizing this intuition yields the following class of CRNs:
\begin{definition}
\emph{Non-competitive CRNs.}
A CRN is \emph{non-competitive} if every species which is decreased in a reaction is a reactant in only that reaction.
\end{definition}
\noindent Note that by the definition above, a reactant may appear in any number of reactions if it is not decreased (e.g., if it acts as a catalyst).

In SI Appendix~\ref{sec:NCCRNs_proofs}, we prove the following about non-competitive CRNs:

\begin{theorem}\label{thm:main}
For non-competitive CRNs, if $\vec{a} \rightarrow \vec{b}$ and $\vec{b}$ is a static state, then
for any state $\vec{a}'$ such that $\vec{a} \rightarrow \vec{a'}$, $\vec{a'}$ converges to $\vec{b}$ for any rate constants under mass-action kinetics.
\end{theorem}
\noindent
Figure~\ref{fig:nc_theorem} illustrates a small application of this theorem.
The precondition of this theorem, that $\vec{a} \rightarrow \vec{b}$ with $\vec{b}$ static, is the same as providing a line-segment path from the input state to a static state with the correct output.
(For the $\max$ example, the line-segment path is simply to apply the reactions maximally in order.)
Thus, this theorem greatly simplifies the analysis of equilibrium for non-competitive CRNs.
Further, the theorem states that any state stoichiometrically compatible with the initial state still converges correctly under mass-action kinetics.
The path $\vec{a} \rightarrow \vec{a}'$ captures a wide class of perturbations, allowing any adversarial conditions to be applied to the system initially, such as non-well-mixedness or withholding of certain reactions, as long as stoichiometry is still obeyed.
Then, as long as mass-action kinetics are allowed to take over, the system converges to the output state $\vec{b}$.
(Note that $\vec{a}'$ can be equal to $\vec{a}$, since $\vec{a} \rightarrow \vec{a}$, meaning that this theorem also implies convergence from the initial state.)

In fact, we can apply Theorem~\ref{thm:main} to rate laws more general than mass action:
\begin{definition}\label{def:fair_rate_law}
A \emph{fair} rate law is any kinetic rate law which satisfies: (1) at any time, the rate of a reaction is nonzero if all of its reactants have nonzero concentration, and
(2) if $\vec{b}$ can be reached from $\vec{a}$ according to the rate law, then  $\vec{a} \rightarrow \vec{b}$.
\end{definition}

\noindent Theorem~\ref{thm:main} holds for any fair rate law.
In~\cite{chen2014rate}, it is proven that mass-action kinetics is fair.
(Note that only item $(2)$ of Definition~\ref{def:fair_rate_law} is nontrivial.)
One only needs to prove their relevant kinetic model has a fair rate law in order to apply Theorem~\ref{thm:main}.
\todom{Possibly mention that MM, Hill-function are fair.}

By the end of this section, we will see that restricting stoichiometrically defined computation to the non-competitive subclass does not restrict computational power.

\subsection{Composition of CRNs}
To construct large programs out of smaller ones requires \emph{composability}: CRNs computing functions $f_1$ and $f_2$ should be straightforwardly concatenated so that $f_2 \circ f_1$ is computed.
However, some of the constructions described do not satisfy composability.
For example, consider composing the $\min$ and $\max$ computing CRNs to compute $z = \min(\max(x_1, x_2), x_3)$ (Figure~\ref{fig:max}).
Based on this failure to compose, we can intuit that a CRN's output species must not be a reactant for a CRN to be composable:

\begin{definition}\label{def:composable}
\emph{Composability.} A CRN is composable if its output species $Y_1,\dots, Y_n$ do not appear as reactants.
\end{definition}

Previous work~\cite{chalk2019composable} proves that this composability definition is necessary\footnote{Although CRNs exist which can be composed and do have their output species as reactants in some reactions, \cite{chalk2019composable} proves that these CRNs can easily be simplified to CRNs which do not have their outputs as reactants.} and sufficient to compose stoichiometrically-defined CRN computations.
Further, they prove that the functions computable while obeying this constraint must be \emph{superadditive}:

\begin{definition}
\emph{Superadditive.} A function $f : \mathbb{R}^n_{\geq 0} \mapsto \mathbb{R}^m_{\geq 0}$ is superadditive if and only if for all $\vec{x}, \vec{y} \in \mathbb{R}^n_{\geq 0}$, $f(\vec{x} + \vec{y}) \geq f(\vec{x}) + f(\vec{y})$.
\end{definition}
Superadditivity is a very strong restriction; for example, the $\max$ function is not superadditive, and so cannot be computed by a composable CRN.
However, an alternative method for representation of logical values in a CRN avoids the superadditivity restriction for composability and simultaneously allows representation of negative numbers, as we will describe next.
\todoi{DS: Probably we should point out that if a CRN is not composable according to the above definition, then it will not be non-competitive.
There is a strong connection between composability and non-competitiveness that's missing right now. Note that the min(max()) example fails because of competition.
}
\todoi{MV: But MAX CRN is not composable and is non-competitive?}

\subsection{Dual-rail CRN computation}
If we wish to represent a variable $x$ that can take on negative values, we use a \emph{dual-rail} representation, which expresses a value $x$ as a difference in concentration between two species $X^+$ and $X^-$.
There are composable CRNs with dual-rail input/output convention which compute the $\min$ and $\max$ functions (Figure~\ref{fig:drminmax}).
\begin{figure*}[]
  \centering
  \includegraphics[scale=.9]{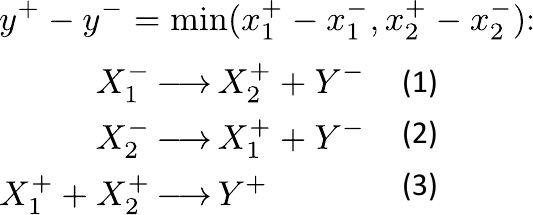}
  \hfill
  \includegraphics[scale = .9]{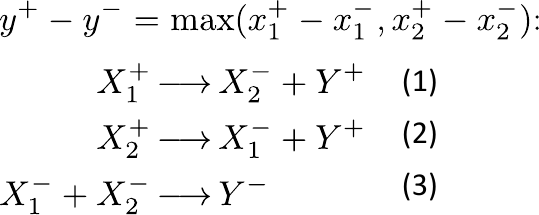}
  \hfill
  \includegraphics[scale=.9]{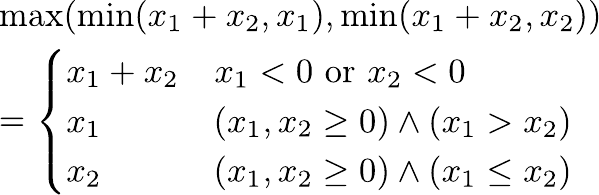}
  \caption{\textbf{Non-competitive, dual-rail, composable CRNs for computing $\min$ (Left) and $\max$ (Middle).}
  To analyze using Theorem~\ref{thm:main}, we can apply reactions maximally, one-by-one, and in order.
  For the $\min$ CRN, applying reaction ($1$) maximally yields $x_2^+ = x_2^+(0) + x_1^-(0)$ and $y^- = x_1^-(0)$; then applying reaction ($2$) yields $x_1^+ = x_1^+(0) + x_2^-(0)$ and $y^- = x_1^-(0) + x_2^-(0)$.
  Then applying reaction ($3$) maximally yields $y^+ = \min(x_1^+, x_2^+)$, and substituting the values of $x_1^+, x_2^+,$ and $y^-$ from applying the first two reactions, we get $y^+ - y^- = \min(x_1^+(0) + x_2^-(0), x_2^+(0) + x_2^-(0)) - (x_1^-(0) + x_2^-(0)) = \min(x_1^+(0) - x_1^-(0), x_2^+(0) - x_2^-(0))$ as desired.
  The $\max$ CRN's correctness follows from a similar analysis.
  \textbf{(Right) Continuous piecewise linear functions are compositions of $\max$, $\min$, and linear functions.} An example application of Theorem~\ref{thm:ovchinnikov} is shown.
  Using the CRNs on the left and middle, along with composable, non-competitive, dual-rail CRNs to compute $y = \frac{p}{q}x$ and $y = x_1 + x_2$, any continuous piecewise linear function can be computed by first applying the transformation of Theorem~\ref{thm:ovchinnikov}.}\label{fig:drminmax}
\end{figure*}

These $\min$ and $\max$ modules are important artifacts related to the computational power of stoichiometrically-defined computation, due to the following theorem.
Continuous piecewise rational linear functions were proven equivalent to expressions which are a $\max$ over $\min$s over rational linear functions (Figure~\ref{fig:drminmax}). 
Formally:
\begin{theorem}\label{thm:ovchinnikov}
Proven in~\cite{ovchinnikov2002max}: For every continuous piecewise linear function $f$ with pieces $f_1,\dots,f_p$, there exists a family $S_1, \dots, S_q \subseteq \{1, \dots, p\}$ with $S_i \not\subseteq S_j$ if $i \neq j$, such that for all $\vec{x}$, $f(\vec{x}) = \max_{i \in {1,\dots,q}}\min_{j \in S_i}f_j(\vec{x})$.
\end{theorem}
\noindent
Rational linear functions are computable, e.g., $qX \rxn pY$ computes $y = \frac{p}{q}x$.
(We will revisit the computation of rational multiplication later in this work, in the context of neural network weight multiplication, and address the issue of using reactions with many reactants which is undesirable.)
Rational affine functions are also computable when the CRN has \emph{initial context} (initial concentrations of non-input species).
Then, the $\min$ and $\max$ modules allow a method for piecewise composition of the rational affine pieces according to Theorem~\ref{thm:ovchinnikov}.
Ultimately, the exact characterization of dual-rail, composable, stoichiometrically-defined CRN computable functions is the set of continuous piecewise rational affine functions~\cite{chen2014rate}.
Further, as we have shown how to compute $\min$, $\max$, and rational affine functions by composable, non-competitive CRNs, we have shown that restricting CRNs to be non-competitive does not restrict computational power.

While at first glance the functions computed seem rather limited since they are composed of rational affine pieces, they indeed can approximate arbitrary curves to any desired accuracy.\todom{CC: Also, figure out whose responsibility it is to explain more about affine.}
Further, their power is underwritten by the empirical power of ReLU neural networks, since such neural networks indeed compute only piecewise rational affine functions.
Thus we motivate the connection between CRNs and ReLU neural networks, and explore this connection in more detail in Sections~\ref{sec:rrelu} and~\ref{sec:brelu}.

\Section{\RReLU: Rational-Weight \ReLU Neural Networks}
\label{sec:rrelu}

\begin{figure*}[!t]
  \centering
  \includegraphics[scale=0.75]{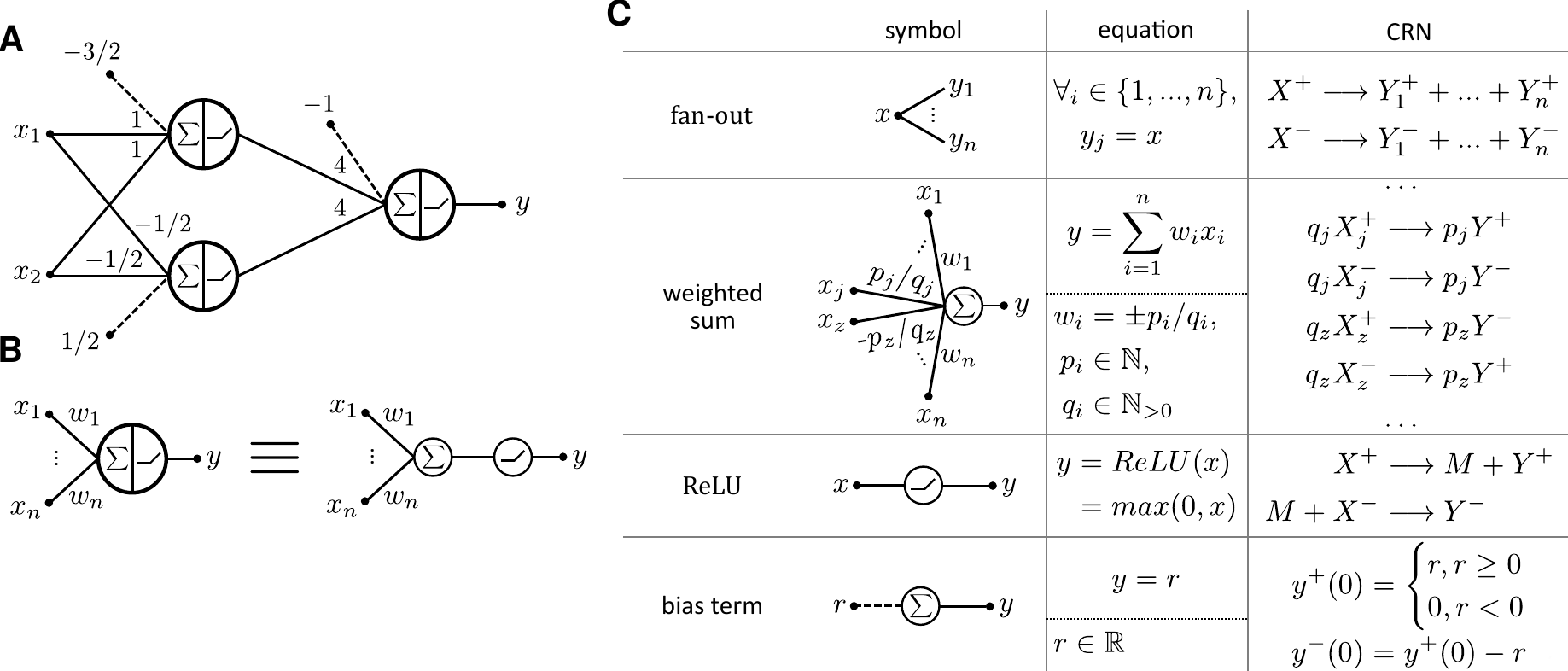}
  \caption{\textbf{CRN implementation of \RReLU neural networks.}
    \textbf{(A)} An example \RReLU network. 
    \textbf{(B)} Decomposition of a neuron into weighted summation and nonlinearity.
    \textbf{(C)} CRN implementation for each \RReLU network component.}
  \label{fig:neural_network}
\end{figure*}

In this and the subsequent section we develop constructions for implementing \ReLU neural networks with stoichiometrically-defined CRNs.
We start with broadly allowing arbitrary rational weights in this section,
and focus on binary weights in Section~\ref{sec:brelu}. 

Rational-Weight \ReLU neural networks (\RReLU) are neural networks with rational weights and \ReLU activation function.
Figure~\ref{fig:neural_network}A shows an example \RReLU neural network.
This network consists of an input layer, a single hidden layer and an output layer with \ReLU activation functions.
The output of the network is defined by: $ {y = \ReLU(\vec{W_2} \cdot \ReLU(\vec{W_1} \cdot \vec{x} + \vec{b_1}) + b_2)} $,
where $\vec{x} \in \mathbb{R}^2$ is an input vector, $\vec{W_1} \in \mathbb{Q}^{2 \times 2}$ is a weight matrix into the hidden layer, $\vec{b_1} \in \mathbb{R}^2$ is a vector of bias terms, $\vec{W_2} \in \mathbb{Q}^{1 \times 2}$ is a weight vector into the output layer with $b_2$ the corresponding bias term, and $y \in \mathbb{R}$ is the output\footnote{We assume all vectors to be column vectors, unless otherwise noted.}:
\[
\vec{W_1}=
  \begin{bmatrix}
    1  & 1 \\
    -1/2 & -1/2
  \end{bmatrix},
\vec{x}=
  \begin{bmatrix}
    x_1 & x_2
  \end{bmatrix}^\top,
\vec{b_1}=
  \begin{bmatrix}
    -3/2 & 1/2
  \end{bmatrix}^\top,
\]
\[
\vec{W_2}=
  \begin{bmatrix}
    4 & 4
  \end{bmatrix},
  b_2 = -1
\]
(Although the inputs and outputs are interpreted as real-value quantities, this particular network happens to compute the XNOR function: $y = \overline{x_1 \oplus x_2}$ if $0$ and $1$ values represent logical False and True.)

Figure~\ref{fig:neural_network}C shows an implementation of such \RReLU networks with composable, non-competitive CRNs.
Note that the different CRN modules (fan-out, weighted sum, and ReLU) are composed in a feedforward manner, where the outputs of the upstream modules are inputs for the downstream modules.
The feedforward structure of the modules allows us to analyze the system module by module, obtaining a path from the initial state to a static state.
We can then apply Theorem~\ref{thm:main}.

Fan-out---passing a value to multiple downstream neurons---is implemented by consuming the input species and producing $n$ output species ($n$ is equal to the fan-out degree), for both positive and negative inputs, as shown in Figure~\ref{fig:neural_network}C.
First apply the first reaction ($X^+ \rxn Y_1^+ + \dots + Y_n^+$) until completion.
This results in $y_i^+=x^+(0)$.
Then apply the second reaction ($X^- \rxn Y_1^- + \dots + Y_n^-$) until completion.
This results in $y_i^-=x^-(0)$, and thus $y_i=y_i^+-y_i^-=x^+(0)-x^-(0)=x(0)$.
Since this is a static state of the fan-out module,
by Theorem~\ref{thm:main} this CRN computes fan-out.

Weighted sum---combining outputs of multiple predecessor neurons by multiplying them with weight (rational number) and summing up the values---is implemented by controlling the stoichiometry of input and output species as shown in Figure~\ref{fig:neural_network}C.
Consider the contribution to the weighted sum by the reaction $q_j X_j^+ \rxn p_j Y^+$.
Running this reaction till completion, $x_j^+(0)$ amount of input is consumed to produce $(p_j/q_j)x_j^+(0)$ amount of the output.
The negative input and output species in reaction $q_j X_j^- \rxn p_j Y^-$ work similarly.
The total contribution to the output species is $(p_j/q_j)(x_j^+(0)-x_j^-(0))=(p_j/q_j)x_j(0)$.
Similar reactions are included for the other input species of the weighted sum (note that positive and negative species are flipped in the case of a negative-signed weight), which results in reaching a static equilibrium where the total contribution to the output species is equal to the weighted sum of the inputs.

\begin{figure}[!t]
  \centering
      \includegraphics[scale = 0.8]{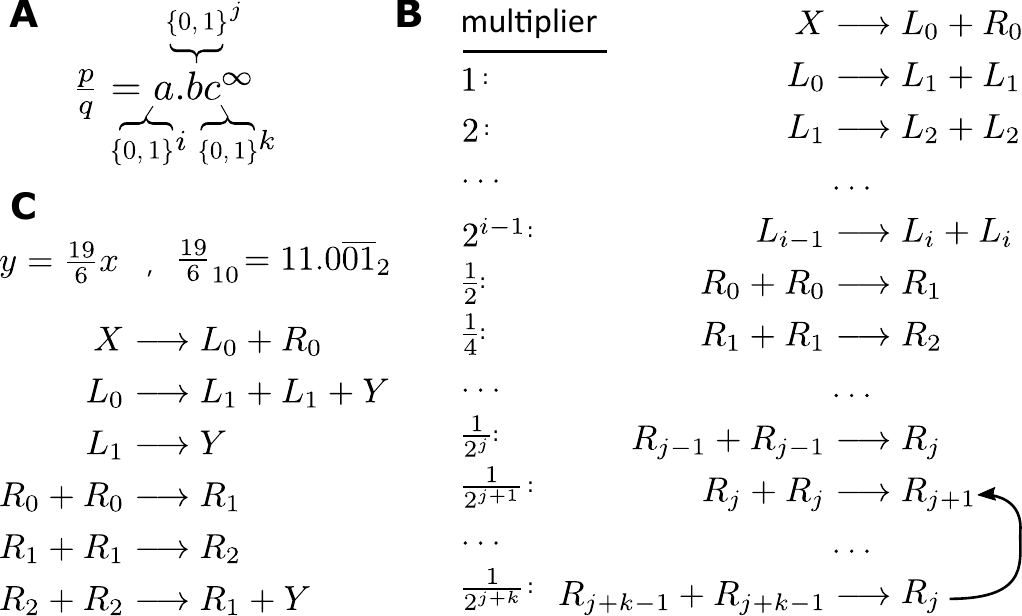}
  \caption{\textbf{Non-competitive bimolecular rational multiplication. (A)} Binary representation of rational $\frac{p}{q}$ where $a$, $b$, and $c$ are binary strings. Strings $a$ and $b$ are length $i$ and $j$, respectively, while $c$ is an infinitely repeated string of length $k$. \textbf{(B)} A scheme for constructing a non-competitive bimolecular CRN for rational multiplication. The reaction chain uses $i+j+k+1$ reactions to ``implement'' the binary expansion. The last reaction creates a ``loop'' in the reaction chain which corresponds to string $c$. The number of times each reaction will be applied (before looping) is some multiple of the initial count of $X$, as indicated by the multiplier column. An output species $Y$ will appear as a product for each reaction where a 1 appears in the binary expansion. \textbf{(C)} An example CRN which computes $y = \frac{19}{6}x$. %
  \emph{Note:} To achieve a dual-rail representation we can repeat this construction twice, for both positive ($X^{+}$, $Y^{+}$) and negative ($X^{-}$, $Y^{-}$) input and output species.}
  \label{fig:rational_crn}
\end{figure}
\todom{CC: In fan-out on Fig.6 part C, the $y_j$ should be $y_i$.}

While rational weight multiplication is easily computable through stoichiometry as above (e.g., $qX \rxn pY$ computes $y = \frac{p}{q}x$), the use of many reactants is undesirable as discussed in Section~\ref{sec:CRNs}.%
We can use the scheme shown in Figure~\ref{fig:rational_crn} for rational weight multiplication using only non-competitive uni- and bimolecular reactions. 
Using reactions of the form $L_0 \rxn L_1 + L_1$ and $R_0 + R_0 \rxn R_1$ we can double and halve the concentration of a species, respectively. 
In this way, a set of reactions may mimic the binary expansion of a given rational $\frac{p}{q}$, generating an output species $Y$ for each $1$ bit in the binary representation. 
If the rational number has an infinitely repeating portion in its binary expansion, our CRN uses a final reaction which ``loops'' back to a previous reaction.
Figure~\ref{fig:rational_crn}c shows a concrete example of this.
A detailed proof of correctness for this construction may be found in SI Appendix~\ref{sec:BRMCRNs_proofs}.
The proof shows a path from a state with concentration $x$ of the input species to a state at static equilibrium with concentration $\frac{p}{q}x$ of the output species.
By Theorem~\ref{thm:main} (and the fact that this CRN is non-competitive), this is sufficient to show that the construction computes $\frac{p}{q}x$.
To satisfy the dual-rail representation, the construction is repeated for both the positive $X^+$ and negative $X^-$ species.
Since this CRN is composable, it may be used for the weighted sum by creating similar reaction chains for all input species.

\ReLU is implemented with two reactions shown in Figure~\ref{fig:neural_network}~\footnote{Enumeration of small CRNs shows that this is the simplest stoichiometrically-defined, composable CRN computing \ReLU in the sense that \ReLU cannot be computed in this manner with fewer than 2 reactions or 5 species~\cite{vasic2020crnsexposed}.}.
We will show a particular line-segment path that leads to a static equilibrium computing \ReLU,
which by Theorem~\ref{thm:main} implies that the CRN computes the \ReLU.
Consider at first applying the first reaction ($X^+ \rxn M + Y^+$) as long as $X^+$ is present.
This results in: $y^+ = m = x^+(0)$ and $x^+=0$.
Then, consider applying the second reaction ($M + X^- \rxn Y^-$) until completion.
The second reaction will execute for $\min(m, x^-)= \min(x^+(0), x^-(0))$.
This results in: $y^- = min(x^+(0), x^-(0))$ and $m = 0 \lor x^- = 0$.
The output of the CRN is then: $y = y^+ - y^- = x^+(0) - \min(x^+(0), x^-(0)) = \max(x^+(0)-x^-(0),0) = \text{ReLU}(x(0))$.
Also, it holds that $x^+=0$ and $m = 0 \lor x^- = 0$;
from which it follows that at least one reactant of both reactions is zero, thus the static equilibrium is reached.
From Theorem~\ref{thm:main} it follows that the CRN computes \ReLU.

Finally, bias terms are implemented by setting the initial concentrations of the corresponding species to the dual-rail value of the bias terms.

To see that the composed modules converge, note that we have shown that each module is composable as in Definition~\ref{def:composable}, and further that since each module is non-competitive, the entire network is non-competititve.
Therefore, applying reactions maximally module-by-module, layer-by-layer gives a straightforward path in the nondeterministic kinetic model from the initial state to a static state with the output equal to the output of the neural network.
Theorem~\ref{thm:main} then argues that the CRN converges correctly under mass-action kinetics or any fair rate law. 
We show an example \RReLU neural network and its complete CRN implementation in SI Appendix~\ref{sec:appendix:rrelu}.

\Section{\BReLU: Binary-Weight \ReLU Neural Networks}
\label{sec:brelu}

Binary-Weight \ReLU neural networks (\BReLU) are neural networks with binary weights ($\pm 1$) and \ReLU activation function.
Since they are a subclass of \RReLU networks, the same translation procedure as illustrated for \RReLU applies.
\BReLU networks were popularized in the machine learning community due to the computational speed-ups they bring (they eliminate the need for a large portion of multipliers which are the most space and power hungry components of specialized deep learning hardware), while at the same time preserving the performance~\cite{courbariaux2015binaryconnect}.
From the angle of CRNs, computing rational weights $\frac{p}{q}\vec{x}$ in dual-rail requires either two reactions with many reactants or many reactions with at most two reactants, neither of which is desirable.\todom{CC: I changed this sentence, since the previous version didn't mention Austin's construction}
Thus, \BReLU networks are a better suited class of neural networks for CRNs than \RReLU, producing CRNs that are easier to implement in a wet lab.
In other words, restriction to binary weights simplifies both silicon- and chemical-hardware implementations of deep learning while maintaining performance.\todom{CC: I added this sentence, please check.}

Note that the fan-out and weighted sum can be merged into a single step since \BReLU networks have $\pm 1$ weights.
Thus, by default, the fan-out and weighted sum of \BReLU networks is implemented using a reaction set similar to the fan-out module in Figure~\ref{fig:neural_network}, with the difference that the $\pm$ signs of the output species are flipped in the case of negative weight.

\Subsection{Translation optimization}\label{sec:optimization}
We find that unimolecular reactions of non-competitive CRNs, such as the first reactions of \ReLU modules, can be eliminated from the CRN by altering the bimolecular reactions and the initial concentrations of the CRN species, a process which we describe next.
Unimolecular reactions are those with exactly one reactant like $A \rxn B + C$.
Whenever $A$ is produced in another reaction, we can replace it with $B + C$.
For example, if there is another reaction $X \rxn A + B$, we replace the reaction with $X \rxn 2B+C$.
Further, we adjust the initial concentrations of the product species ($B$ and $C$) by increasing them by the initial concentrations of the reactant ($A$).
Importantly, this transformation works only if $A$ is not a reactant in any other reaction;
for example, if there were another reaction like $X + A \rxn Y$, it is not clear what to replace instances of $A$ with, and indeed it is not possible to remove the unimolecular reaction in that case.
Luckily, our constructions are non-competitive and we are able to show that for non-competitive CRNs the optimization does not affect the state of convergence (SI Appendix~\ref{sec:optimization_proof}).
The optimization procedure is illustrated in Figure~\ref{fig:optimization}.

\begin{figure*}[!t]
  \centering
  \includegraphics[scale=0.55]{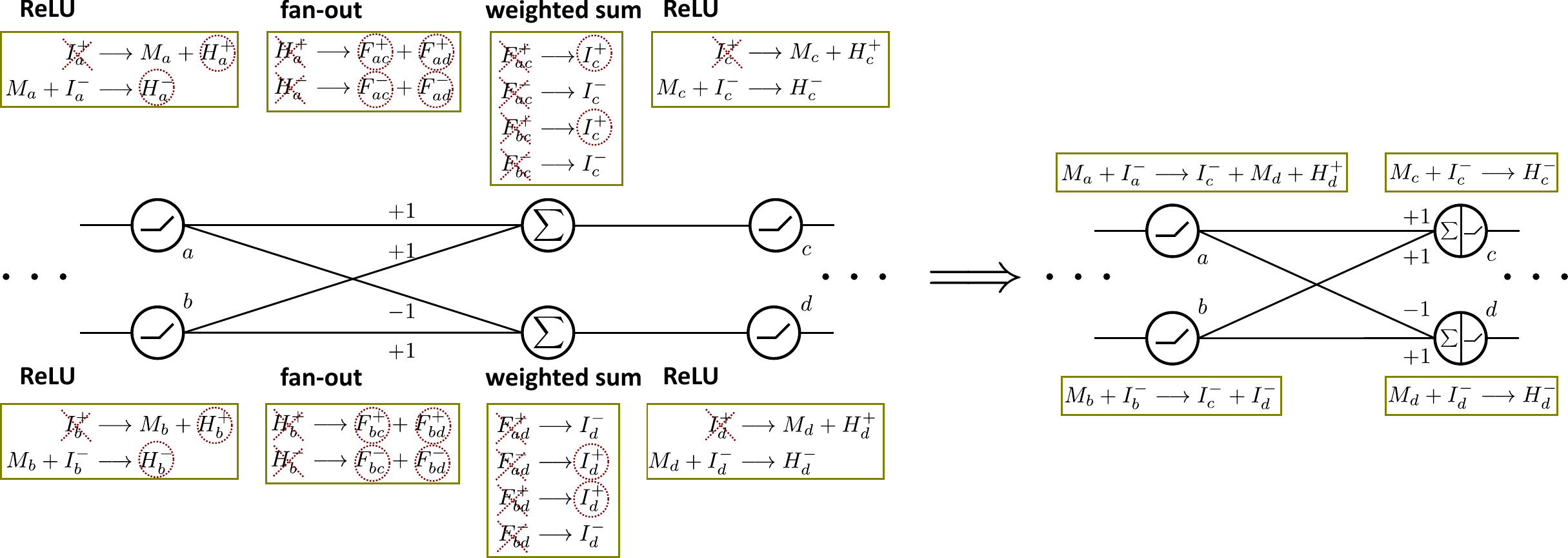}
  \caption{\textbf{CRN optimization procedure.}
    \textbf{(Left)} Neural network and its corresponding CRN before the optimization.
    \textbf{(Right)} Neural network and its corresponding CRN after the optimization.}
  \label{fig:optimization}
\end{figure*}

\RReLU networks allow for the optimization of fan-out modules, partial optimization of \ReLU modules (only the unimolecular reaction) and weighted sum modules only in the cases where the weight denominator is equal to $1$ (integer weights).
\BReLU networks in addition allow optimization of weighted sum modules in all cases.
Note that the unimolecular reactions corresponding to the input species are not optimized in order not to alter the input to the system.
The CRN resulting from the optimization of a \BReLU network thus has the property that there are no unimolecular reactions besides the input layer, for which there are two reactions per input.
In other words, the CRN of a \BReLU network consists of (a) a bimolecular reaction per \RReLU node, and (b) two unimolecular reactions per input of the neural network.

Optimization of some adversarial ReLU networks results in reactions with a number of products exponential in the depth of the network.
Understanding the scaling of the number of products is an important avenue for future work to ensure feasible CRNs.

\begin{figure}[!t]
  \centering
  \includegraphics[scale=0.8]{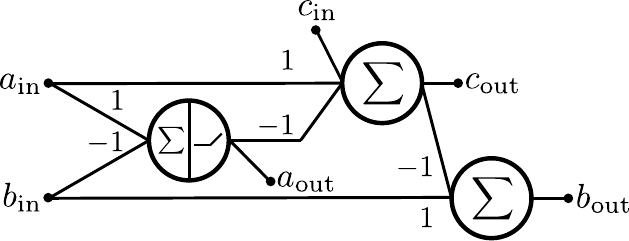}
  \caption{\textbf{A composable binary-weight ReLU network simulating a chemical reaction $A + B \rxn C$.}
  Given any vector $\vec{x}$ of initial concentrations of species $A, B,$ and $C$, the equilibrium state $\vec{b}$ of the reaction $A + B \rxn C$ has $\vec{b}(A) = \vec{a}(A) - \min(\vec{a}(A), \vec{a}(B))$, $\vec{b}(B) = \vec{a}(B) - \min(\vec{a}(A), \vec{a}(B))$, and $\vec{b}(C) = \min(\vec{a}(A), \vec{a}(B)$).}
  \label{fig:chelu_gadget}
\end{figure}

\Subsection{BReLU networks simulate CRNs}
We have seen that non-competitive CRNs can compute any function computed by a \BReLU network where each reaction (except for the input layer reactions) corresponds to one \BReLU node.
One interpretation of this is that CRNs efficiently simulate \BReLU networks.
A natural question is the converse: can any CRN be efficiently simulated by a \BReLU network?
In this subsection we answer this question at least for a subclass of CRNs which we call $\emph{\CheLU}$ networks, showing that they can be simulated by \BReLU networks with one ReLU node per reaction.

First we define a subclass of CRNs as the target to be simulated.
The first restriction is that reactions have at most two reactants (reactions with more than two reactants are anomalous as discussed in Section~\ref{sec:CRNs}).
The second restriction is that the CRN is \emph{feed-forward}. 
This can be formalized by saying that there is a total ordering on reactions such that products of a reaction cannot be reactants of a reaction earlier in the ordering.
The third restriction is that every species appears at most once per reaction.
Intuitively, this restriction is placed because a reaction like $X + X \rightarrow \dots$ essentially halves the signal of $X$, which has no analog in binary-weight neural networks.
Lastly, we restrict the CRNs to be non-competitive.
For their connection to \BReLU networks, we call this class of CRNs \emph{\CheLU networks}.

We next define what is meant by simulation of \CheLU networks by \BReLU networks.
Of course, \BReLU networks have no sense of kinetics or dynamics.
For this reason we disregard kinetics and instead focus on initial and equilibrium states of the \CheLU network, and mapping those states to inputs and outputs of a \BReLU network.
Formally, if a CRN has one equilibrium state, we say a ReLU neural network simulates that CRN if, for all initial states $\vec{a}$, the equilibrium state $\vec{b}$ given $\vec{a}$ is equal to the output vector of the ReLU neural network given $\vec{a}$ as input.

We give a small, composable \BReLU network (Figure~\ref{fig:chelu_gadget}) which simulates a single \CheLU reaction.
Composing this small network to simulate larger \CheLU networks is straightforward since we restrict \CheLU networks to be feed-forward.
The \BReLU network uses one ReLU node and two summation nodes per reaction, although the summation nodes can be removed with the clever addition of more edges to achieve one ReLU node per reaction.

Thus, \BReLU networks and \CheLU networks simulate each other, one node per reaction and vice versa, and so efficient networks in one model transfer to the other.
Although \CheLU networks at first seem restricted, the empirical power shown of \BReLU networks implies that \CheLU networks are a rich and powerful class of CRNs, whose restrictions make them easy targets for implementation by synthetic means.

\Section{Simulations}
In this section we describe numerical experiments showcasing compilation from \BReLU neural networks to CRNs.
We train \BReLU networks on IRIS~\cite{fisher1936use,anderson1936species}, MNIST~\cite{lecun1998gradient}, virus infection~\cite{GSE73072}, and pattern formation datasets.
We translate trained neural networks to CRNs following our compilation technique (Figure~\ref{fig:neural_network}),
and simulate the reactions' behavior under mass-action kinetics using an ODE simulator~\cite{CRNSimulatorPackage}.
Our main goal is to show the equivalence of a trained neural network and compiled CRN, and not to improve accuracy of ML models, which is orthogonal to our work.

\Subsection{IRIS}
\hfill

\textbf{Dataset}.
The IRIS dataset consists of $150$ examples of $3$ classes of flowers (Setosa, Versicolor or Virginica), and $4$ features per example (sepal length and width, and petal length and width).
Considering a small dataset size ($150$ examples), and that our primary goal is to show the equivalence of a neural network and the compiled CRN, we train and evaluate on the whole IRIS dataset.

\textbf{Results}.
We train a neural network with a single hidden layer consisting of $3$ units, $4$ input units (capturing the features of IRIS flowers), and $3$ output units where the unit with the highest value determines the output class.
We achieve accuracy of $98$\% ($147$ out of $150$ examples correctly classified) with a trained BinaryConnect neural network.
In the resulting network, $5$ weights out of $21$ total weights are zero-valued.
\todom{DS: Explain why zero weights are important.}
We translate the network to the equivalent CRN consisting of $18$ chemical reactions (unoptimized compilation), or $9$ chemical reactions (optimized compilation).
We simulate both versions of CRNs and confirm that their outputs (labels) match the outputs of the neural network in all of the $150$ examples.

\Subsection{MNIST}
\hfill

\textbf{Dataset}.
The MNIST dataset consists of labeled handwritten digits, where features are image pixels, and labels are digits ($0$ to $9$). 
We split the original MNIST training set consisting of $60,000$ images into $50,000$ for the training set, and $10,000$ for the validation set.
We use the original test set consisting of $10,000$ images.
In a preprocessing stage we center the images (as done in the BinaryConnect work). 
Additionally, aiming at a smaller neural network and CRN, we scale the images down from $28 \times 28$ to $14 \times 14$.

\textbf{Results}.
We train a neural network with one hidden layer of $64$ units.
The neural network has $14^2$ input units (one per pixel), and we use $10$ output units (for digits $0$ to $9$). 
We train the neural network to maximize the output unit corresponding to the correct digit.
Our model achieves accuracy of $93.92$\% on the test set,
In the resulting model $23$\% of weights are zero.
Note that we did not focus on achieving high accuracy; BinaryConnect in original paper achieves accuracy of over $98$\%,
but uses more hidden layers and units ($3$ layers with $1024$ units each).
Instead we used fewer units in order to produce a smaller neural network and CRN.
We translate the network to an equivalent CRN consisting of $648$ chemical reactions (unoptimized compilation), and $456$ chemical reactions (optimized compilation).
The CRN consists of $2 \cdot 14^2$ input species (two species per input unit encoding positive and negative parts),
and similarly $2 \cdot 10$ output species.
We simulate the CRN on $100$ randomly chosen examples from the test set, and confirm that output matches that of the neural network in all of the cases.

\Subsection{MNIST Subset}
\hfill

\textbf{Dataset}.
With a goal of creating a smaller network we trained a model on a subset of the MNIST dataset (only digits $0$ and $1$).

\textbf{Results}.
We train a network with $1$ hidden layer with $4$ units.
We now scaled images to $8 \times 8$, using a neural network with $8^2$ input units and $2$ output units.
Our model achieves accuracy of $98.82$\% on the test set.
In the resulting model $23$\% of weights are set to zero.
The resulting CRN consists of $140$ reactions (unoptimized compilation), and $128$ reactions (optimized compilation).

\Subsection{Virus Infection}
\hfill

\textbf{Dataset}.
For the virus infection classifier, we used data from NCBI GSE73072~\cite{GSE73072}.
The dataset contains microarray data capturing human gene expression profiles, with the goal of studying four viral infections: H1N1, H3N2, RSV, and HRV (labels).
There are $148$ patients in the dataset, each with about $20$ separate profiles taken at different times during their infection period, for a total of $2,886$ samples.
The dataset contains information about which patient was infected and during which point of time.
We filter the samples leaving only those that correspond to an active infection, and thus make the data suitable for classification of the four viruses.
Finally, we have in total $698$ examples, split in $558$ for training, $34$ for validation, and $104$ for testing.
Each sample measures expression of $12,023$ different genes (features); we use the $10$ most relevant genes as features which are selected using the GEO2R tool~\cite{GEO2R} from the NCBI GEO.

\textbf{Results}.
We train a neural network with one hidden layer with $8$ units, $10$ input units capturing the expression of different genes, and $4$ output units classifying between virus infections.
We achieve test set accuracy of $98.08$\%.
In the resulting model $32$\% of weights are zero.
We translate the network to the equivalent CRN consisting of $52$ chemical reactions (unoptimized compilation), or $28$ chemical reactions (optimized compilation).
We simulate the CRN on $100$ randomly chosen examples from the test set, and confirm that output matches that one of the neural network in all of the cases.

\Subsection{Pattern Formation}
\hfill

\begin{figure*}[!t]
  \centering
  \includegraphics[scale=0.75]{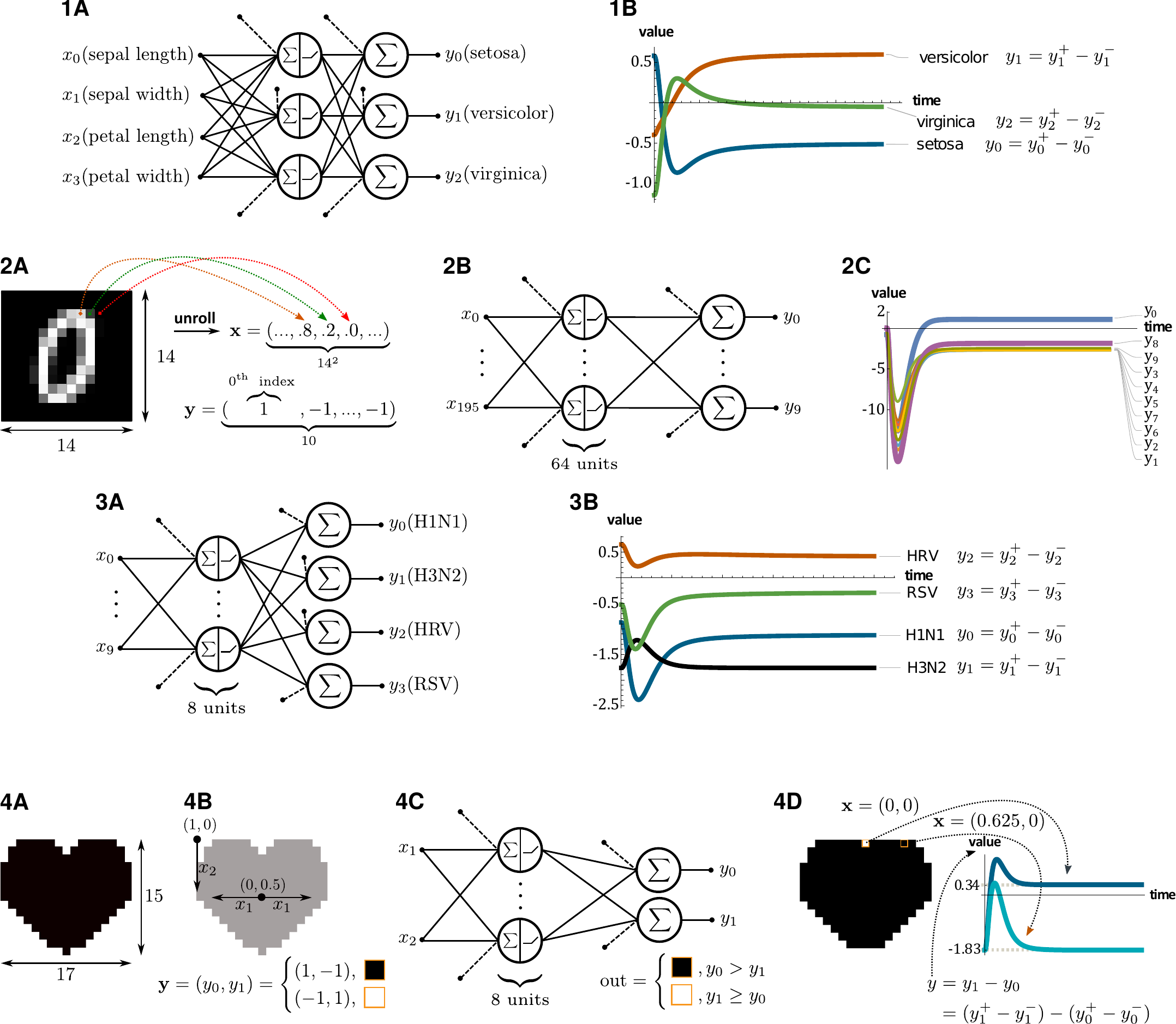}
  \caption{\textbf{Neural network architecture, input/output encoding, and CRN simulations for different datasets.}
    \textbf{(1)} IRIS.
    \textbf{(1A)} Neural network architecture.
    \textbf{(1B)} CRN simulation results.
    \textbf{(2)} MNIST.
    \textbf{(2A)} Input image and its input/output encoding.
    Each image from the MNIST dataset is unrolled into a vector, and the output label is represented as a 10D vector.
    \textbf{(2B)} Neural network architecture.
    \textbf{(2C)} CRN simulation results for the input shown in 2A.
    \textbf{(3)} Virus Infection.
    \textbf{(3A)} Neural network architecture.
    \textbf{(3B)} CRN simulation results.
    \textbf{(4)} Pattern formation.
    \textbf{(4A)} Image used to construct a dataset.
    \textbf{(4B)} Input and output encoding for a position (pixel) in the input image.
    An input is encoded using 2D coordinates: ($x_1$) symmetric horizontal coordinates (starting in the image center) and ($x_2$) vertical coordinates starting from the top left edge of the image.
    An output, which can be either black or white pixel, is encoded as a 2D vector as shown in the figure.
    \textbf{(4C)} Neural network architecture.
    \textbf{(4D)} Image learned by the neural network, and CRN simulation results for $2$ input values (positions).
  }
  \label{fig:subjects}
\end{figure*}

\textbf{Dataset}.
We construct the dataset from the image shown in Figure~\ref{fig:subjects}.
For each pixel, we create a training example with $(x_1,x_2)$ coordinates as input and a label representing value of the pixel.
Input $x_1$ represents the horizontal distance from the center of the image, and $x_2$ represents the vertical distance from the top left corner of the image.
The value of the label is $0$ if the pixel is black and $1$ if white.
The dimensions of the figure are $17 \times 15$; thus there are $255$ examples in the dataset.

\textbf{Results}.
We train a neural network with one hidden layer containing $8$ units, $2$ input units for specifying the location in the coordinate system, and $2$ output units classifying the input location (pixel) as a black or white.
We achieve test set accuracy of $98.44$\% ($4$ out of $255$ pixels are misclassified).
Both original and learned image are shown in Figure~\ref{fig:subjects}.
Note that test and training set are same, as the goal in this task is to overfit to the training set (image).
In the resulting model $15.62$\% of weights are set to zero ($5$ out of $32$ weights).
We translate the network to the equivalent CRN consisting of $36$ reactions (unoptimized compilation), and $13$ reactions (optimized compilation).
We simulate the CRN on all inputs, and confirm that output matches that one of the neural network in all of the cases.

\Subsection{Training Specifics}
\hfill

We use the implementation of BinaryConnect networks published by the authors of the original work~\cite{courbariaux2015binaryconnect},
and follow the same training procedure except for the following:
(1) We focus solely on the \ReLU activation function since other activation functions such as sigmoid, hyperbolic tangent, and softmax are not continuous piecewise linear and thus cannot be implemented with rate-independent CRNs~\cite{chen2014rate}.
(2) We add support for $0$ weights by discretizing the real valued weight to zero if it is in the range $[-\tau, \tau]$; where for $\tau$ we used $0.15$.
(3) We do not use batch normalization~\cite{ioffe2015batch}.
Batch normalization would incur multiplication and division operations at the inference stage (training stage is not a problem) that would be hard to efficiently implement in CRNs.
Instead, we rely on Dropout~\cite{srivastava2014dropout} (stochastically dropping out units in a neural network during training) as a regularization technique.
In all our experiments we use the square hinge loss (as used in BinaryConnect) with ADAM optimizer.

We train on IRIS dataset for $10,000$ epochs, batch size $16$ and return the best performing epoch.
We train on MNIST dataset for $250$ epochs, batch size $100$, measuring the validation accuracy at each epoch, and returning the model that achieved the best validation accuracy during training.
For the MNIST subset dataset we use same number of epocs and batch size.
We train on the virus infection dataset for $200$ epochs, batch size $16$, and return the model that achieved the best validation set accuracy.
We train on the pattern formation dataset for $50,000$ epochs, and batch size of $255$.
We use an exponentially decaying learning rate.
The rate constants of all reactions are set to $1$, and all chemical simulations are performed for $50$ arbitrary time units in the CRNSimulator package~\cite{CRNSimulatorPackage}.

\Section{Related Work}

A brief conference version of this work focused on the binary-weight ReLU network implementation~\cite{vasic2020deep}.
In this full version, we introduce the machinery of non-competitive CRNs allowing for proofs of correctness, the general construction for rational weight ReLU networks, and the inverse construction showing simulation of CRNs by ReLU networks.

\todom{cite CRNs exposed paper}
Prior work has studied a number of properties of CRNs that arise from stoichiometry alone and are independent of rates~\cite{clarke1988stoichiometric,feinberg2019foundations}.
In the context of using CRNs to perform computation, computation by stoichiometry~\cite{chen2014rate} was directly motivated by the notion of stable computation in population protocols~\cite{angluin2006computation}.
Other notions of nearly rate-independent computation involved a coarse separation into fast and slow reactions~\cite{senum2011rate}.

\todom{DS: Point out that rational multiplication is not feed-forward even according to feed-forward definition in revised rate-independent paper. [BUT: Since the feed-forward part isn't published yet, we can't say anything about it yet. Leave todo for later.]}

Recent work took a different but related approach to formalizing and verifying rate independence~\cite{degrand2020graphical}. 
They considered a broad class of rate functions and identified three easy-to-check conditions that force convergence to the same point under any rate function in this class. 
Specifically, they showed that it is sufficient for the CRN to be synthesis-free,
loop-free,
and fork-free.
The first condition means that every reaction decreases some species, the second condition is equivalent to our feedforward condition, and the last is a more restricted version of non-competition.
\todom{When we formally define non-competitive, we should explicitly contrast it with this paper}
\todom{We should also point out that our conditions are not necessary.}
Although most of the constructions in this paper satisfy the above conditions, 
our construction for implementing rational multiplication with bimolecular reactions (Fig.~\ref{fig:rational_crn}) does not satisfy the loop-free (feedforward) condition and is thus not amenable to this analysis.

The connection between CRNs and neural networks has a long history.
It has been observed that biological regulatory networks may behave in manner analogous to neural networks.
For example, both phosphorylation protein-protein interactions~\cite{hellingwerf1995signal,bray1995protein} and transcriptional networks~\cite{buchler2003schemes} can be viewed as performing neural network computation.
\todom{[...]}
Hjelmfelt et al~\cite{hjelmfelt1991chemical} proposed a binary-valued chemical neuron, whose switch-like behavior relies on competition between excitation and inhibition. 
More recently, Moorman et al~\cite{moormandynamical2019} proposed an implementation of \ReLU units based on a fast bimolecular sequestration reaction which competes with unimolecular production and degradation reactions. 
Recently, Anderson et al~\cite{anderson2020reaction} developed a different mass-action CRN for computing the ReLU and smoothed ReLU function.

In contrast to the prior work, 
our implementation relies solely on the stoichiometric exchange of reactants for products, and is thus completely independent of the reaction rates.
Our CRN is also significantly more compact, 
using only a single bimolecular reaction per neuron, with two species per every connection (without any additional species for the neuron itself).

We use neural networks as a way to program chemistry.
The programming is done offline in the sense that neural networks are trained in silico.
However, there is a body of work on creating chemical systems that are capable of learning in chemistry~\cite{chiang2015reconfigurable,blount2017feedforward}.
Although these constructions are much more complex than ours, and arguably difficult to realize, they demonstrate the proof-of-principle that chemical interactions such as those within a single cell are capable of brain-like behavior.

Besides the above mentioned theoretical work on chemical neural networks,
wet-lab demonstration of synthetic chemical neural computation argues that 
the theory is not vapid and that neural networks could be realized in chemistry.
A chemical linear classifier reading gene expression levels could perform basic disease diagnostics~ 
\cite{lopez2018molecular}.
Larger systems based on strand displacement cascades were used to implement Hopfield associative memory~\cite{qian2011neural},
and winner-take-all units to classify MNIST digits~\cite{cherry2018scaling}.
Interestingly, the direct strand displacement implementation of a neuron by our construction is significantly simpler (in terms of the number of components needed) than the previous laboratory implementations, arguing for its feasibility.

\todoi{To add:
Discrete CRNs, connect to Winfree probabilistic inference}

\Section{Conclusion}

While computation in CRNs typically depends on reaction rates, 
rate-independent information processing occurs in the stoichiometric transformation of reactions for products.
In order to better program such computation, we advance non-competition as a useful property, allowing us to analyze an infinite continuum of possible, highly parallel trajectories via a simple sequential analysis.
We further demonstrate embedding complex information processing in such rate-independent CRNs by mimicking neural network computation.
For binary weight neural networks, 
our construction is surprisingly compact in the sense that we use exactly one reaction per ReLU node.
This compactness argues that neural networks may be a fitting paradigm for programming rate-independent chemical computation.

As proof of principle, we demonstrate our scheme with numerical simulations of traditional machine learning tasks (IRIS and MNIST), 
as well as tasks better aligned with potential biological applications (virus identification and pattern formation).
The last two examples rely on chemically available information for input, and thus argue for the potential biological and medical utility of programming chemical computation via a translation from neural networks.

While numerical simulations confirm convergence to the correct output, 
further work is needed to study the speed of convergence.
How does the speed vary with the complexity and structure of the CRN and the corresponding neural network?
As an example of how such convergence speed might be analyzed,
prior work showed that, e.g., $90\%$-completion time scales quadratically with the number of layers in the network if it logically represents a tree of bimolecular reactions~\cite{seelig2009time}.

Although in principle arbitrary CRNs can be implemented using DNA strand displacement reactions,
current laboratory demonstrations have been limited to small systems~\cite{srinivas2017enzyme},
and many challenges remain in constructing large CRNs in the laboratory.
Rate independent CRNs possibly offer an attractive implementation target due to their absolute robustness to reaction rates.

Only three kinds of computing hardware are currently widespread: electronic computers, living brains, and chemical regulatory networks, the last occurring within every cell in every living organism. 
Given the society-changing success of electronic computers and the recent neural networks revolution inspired by computation in the brain,
it may be argued that chemical computation is the least understood of the three.
Upon the refinement of theoretical principles and experimental methods, 
the impact of chemical computation could be felt in far-reaching ways in synthetic biology, medicine, and other fields.
Chemical computation by stoichiometry, and methods of programming and training such computation developed here, provide a distinct approach to bottom-up engineering of molecular information processing.

\todoi{DS: Somewhere add the possibility of adding cancellation reactions in the middle of the circuit so that $X^+$ and $X^-$ separately don't get too large [cite Georg's paper when it comes out]} %

\acknow{This work was supported by NSF grant CCF-1901025 to DS, and CCF-1718903 to SK. We thank David Doty and Erik Winfree for essential discussions.}

\showacknow{} %

\bibliography{bib}

\clearpage
\section*{Supplementary Information Appendix}
\label{sec:appendix}
\Subsection{\RReLU example}
\label{sec:appendix:rrelu}

\begin{figure}[!t]
  \centering
  \begin{subfigure}[!t]{0.5\textwidth}
    \centering
    \includegraphics[scale=0.8]{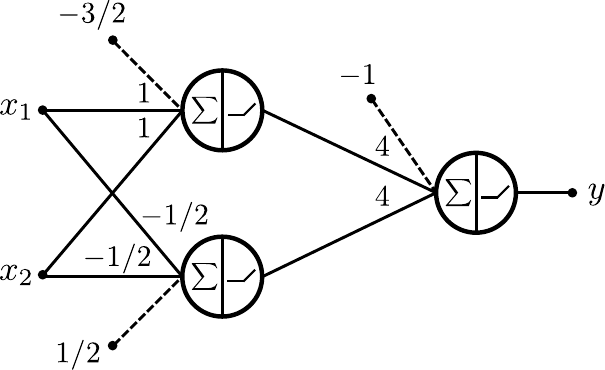}
    \caption{\RReLU neural network.}
    \label{fig:rrelu-example}
  \end{subfigure}
  \begin{subfigure}[!t]{0.5\textwidth}
    \centering
    \small
    \begin{align}
      X_1^+ &\rxn F_{1,1,1}^+ + F_{1,1,2}^+ \\
      X_1^- &\rxn F_{1,1,1}^- + F_{1,1,2}^- \\
      X_2^+ &\rxn F_{1,2,1}^+ + F_{1,2,2}^+ \\
      X_2^- &\rxn F_{1,2,1}^- + F_{1,2,2}^- \\
      F_{1,1,1}^+ &\rxn I_{1,1}^+ \\
      F_{1,1,1}^- &\rxn I_{1,1}^- \\
      F_{1,1,2}^+ + F_{1,1,2}^+ &\rxn I_{1,2}^- \\
      F_{1,1,2}^- + F_{1,1,2}^- &\rxn I_{1,2}^+ \\
      F_{1,2,1}^+ &\rxn I_{1,1}^+ \\
      F_{1,2,1}^- &\rxn I_{1,1}^- \\
      F_{1,2,2}^+ + F_{1,1,2}^+ &\rxn I_{1,2}^- \\
      F_{1,2,2}^- + F_{1,1,2}^- &\rxn I_{1,2}^+ \\
      I_{1,1}^+ &\rxn M_{1,1} + H_{1,1}^+ \\
      M_{1,1} + I_{1,1}^- &\rxn H_{1,1}^- \\
      I_{1,2}^+ &\rxn M_{1,2} + H_{1,2}^+ \\
      M_{1,2} + I_{1,2}^- &\rxn H_{1,2}^- \\
      H_{1,1}^+ &\rxn F_{2,1,1}^+ \\
      H_{1,1}^- &\rxn F_{2,1,1}^- \\
      H_{1,2}^+ &\rxn F_{2,2,1}^+ \\
      H_{1,2}^- &\rxn F_{2,2,1}^- \\
      F_{2,1,1}^+ &\rxn 4I_{2,1}^+ \\
      F_{2,1,1}^- &\rxn 4I_{2,1}^- \\
      F_{2,2,1}^+ &\rxn 4I_{2,1}^+ \\
      F_{2,2,1}^- &\rxn 4I_{2,1}^- \\
      I_{2,1}^+ &\rxn M_{2,1} + Y^+ \\
      M_{2,1} + I_{2,1}^- &\rxn Y^-
    \end{align}
    \caption{CRN implementation of the \BReLU neural network.}
  \end{subfigure}
  \begin{subfigure}[!t]{0.5\textwidth}
    \centering
    \small
    \begin{align}
      i_{1,1}^-(0) &= 3/2 \\
      i_{1,2}^+(0) &= 1/2 \\
      i_{2,1}^-(0) &= 1
    \end{align}
    \caption{Initial concentrations implementing bias terms of the \RReLU neural network.}
  \end{subfigure}
  \caption{Example \RReLU network and its CRN counterpart.}
  \label{fig:rrelu-crn}
\end{figure}
  
\begin{figure}[!t]
  \centering
  \begin{subfigure}[!t]{0.5\textwidth}
    \centering
    \small
    \begin{align}
      X_1^+ &\rxn M_{1,1} + 4M_{2,1} + 4Y^+ + F_{1,1,2}^+ \\
      X_1^- &\rxn I_{1,1}^- + F_{1,1,2}^- \\
      X_2^+ &\rxn M_{1,1} + 4M_{2,1} + 4Y^+ + F_{1,2,2}^+ \\
      X_2^- &\rxn I_{1,1}^- + F_{1,2,2}^- \\
      F_{1,1,2}^+ + F_{1,1,2}^+ &\rxn I_{1,2}^- \\
      F_{1,1,2}^- + F_{1,1,2}^- &\rxn M_{1,2} + 4M_{2,1} + 4Y^+ \\
      F_{1,2,2}^+ + F_{1,1,2}^+ &\rxn I_{1,2}^- \\
      F_{1,2,2}^- + F_{1,1,2}^- &\rxn M_{1,2} + 4M_{2,1} + 4Y^+ \\
      M_{1,1} + I_{1,1}^- &\rxn 4I_{2,1}^- \\
      M_{1,2} + I_{1,2}^- &\rxn 4I_{2,1}^- \\
      M_{2,1} + I_{2,1}^- &\rxn Y^-
    \end{align}
    \caption{Optimized CRN implementing the \RReLU neural network from Figure~\ref{fig:rrelu-example}.}
  \end{subfigure}
  \begin{subfigure}[!t]{0.5\textwidth}
    \centering
    \small
    \begin{align}
      i_{1,1}^-(0) &= 3/2 \\
      i_{2,1}^-(0) &= 1 \\
      m_{1,2}(0) &= 1/2 \\
      m_{2,1}(0) &= 2 \\
      y^+(0) &= 2
    \end{align}
    \caption{Initial concentrations after the optimization.}
  \end{subfigure}
  \caption{Optimized CRN implementing \RReLU neural network.}
  \label{fig:rrelu-crn-optimized}
\end{figure}

Figure~\ref{fig:rrelu-crn} shows a full implementation of an \RReLU network, and Figure~\ref{fig:rrelu-crn-optimized} shows the CRN after optimization procedure is performed.

\Subsection{\BReLU example}
\label{sec:appendix:brelu}

\begin{figure}[!t]
  \centering
  \begin{subfigure}[!t]{0.5\textwidth}
    \centering
    \includegraphics[scale=0.8]{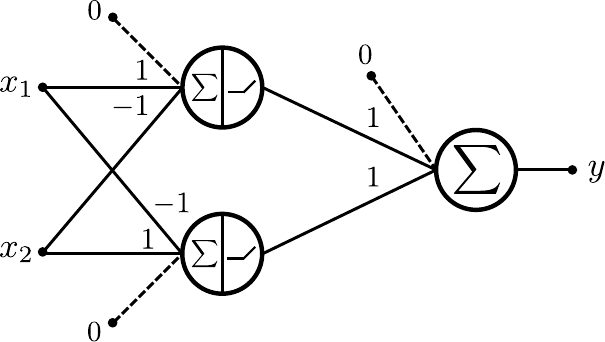}
    \caption{\BReLU neural network.}
  \end{subfigure}
  \begin{subfigure}[!t]{0.5\textwidth}
    \centering
    \small
    \begin{align}
      X_1^+ &\rxn I_{1,1}^+ + I_{1,2}^- \\
      X_1^- &\rxn I_{1,1}^- + I_{1,2}^+ \\
      X_2^+ &\rxn I_{1,1}^- + I_{1,2}^+ \\
      X_2^- &\rxn I_{1,1}^+ + I_{1,2}^- \\
      I_{1,1}^+ &\rxn M_{1,1} + H_{1,1}^+ \\
      M_{1,1} + I_{1,1}^- &\rxn H_{1,1}^- \\
      I_{1,2}^+ &\rxn M_{1,2} + H_{1,2}^+ \\
      M_{1,2} + I_{1,2}^- &\rxn H_{1,2}^- \\
      H_{1,1}^+ &\rxn Y^+ \\
      H_{1,1}^- &\rxn Y^- \\
      H_{1,2}^+ &\rxn Y^+ \\
      H_{1,2}^- &\rxn Y^-
    \end{align}
    \caption{CRN implementation of the \BReLU neural network.
      This CRN is partially optimized -- only fan-out module is optimized.}
  \end{subfigure}
  \begin{subfigure}[!t]{0.5\textwidth}
    \centering
    \small
    \begin{align}
      X_1^+ &\rxn M_{1,1} + Y^+ + I_{1,2}^- \\
      X_1^- &\rxn I_{1,1}^- + M_{1,2} + Y^+ \\
      X_2^+ &\rxn I_{1,1}^- + M_{1,2} + Y^+ \\
      X_2^- &\rxn M_{1,1} + Y^+ + I_{1,2}^- \\
      M_{1,1} + I_{1,1}^- &\rxn Y^- \\
      M_{1,2} + I_{1,2}^- &\rxn Y^- \\
    \end{align}
    \caption{Optimized CRN implementation of the \BReLU neural network.}
  \end{subfigure}
  \caption{Example \BReLU network and its CRN counterpart.}
  \label{fig:brelu-crn}
\end{figure}

Figure~\ref{fig:brelu-crn} shows a full implementation of an \BReLU network.
\subsection{Proof of Theorem~\ref{thm:main}}\label{sec:NCCRNs_proofs}
Here we prove that if a non-competitive CRN can reach a static state in the nondeterministic kinetic model, then the CRN converges to that state under any fair rate law.
This idea simplifies the proof of a non-competitive CRNs' convergence to the simple task of identifying one path to a static state.
Notably, the results here simplify proofs of convergence for constructions given in~\cite{chen2014rate, chalk2019composable}.

To prove Theorem~\ref{thm:main}, several lemmas are provided along the way.
This first lemma does most of the work, showing that line-segment reachability of $\vec{a} \rightarrow \vec{b}$ with $\vec{b}$ a static state places severe restriction on the possible paths leaving $\vec{a}$.
Note that a path from $\vec{a}$ to $\vec{b}$ refers to the sequence of straight-line reachability relations $\rightarrow^1_\vec{u_i}$ which show that $\vec{a} \rightarrow \vec{b}$.

\begin{lemma}\label{lem:nc_implies_two_paths_flux_difference_implies_one_is_not_static}
Assume a CRN is non-competitive.
Consider two paths $p_1$ and $p_2$ leaving state $\vec{a}$.
If $p_1$ has finite length and ends in state $\vec{b}$ and $p_2$ applies some reaction $R$ more than $p_1$, then $\vec{b}$ is not static.
\end{lemma}
\begin{proof}
First, some notation: if $\vec{x} \rightarrow^1_{\vec{w}_1}\dots\rightarrow^1_{\vec{w}_n} \vec{y}$, then we use this shorthand notation for the sum of the flux of a reaction $R$ along the path: $F_{\vec{x}\rightarrow \vec{y}}(R) = \sum_{i=1}^n \vec{w}_i(R)$.

Write path $p_2$ as $\vec{a} \rightarrow_{\vec{u_1}}^1 \vec{a_1} \rightarrow_{\vec{u_2}}^1 \dots$.
Choose the minimal $i$ such that $\vec{a_i}$ satisfies that there exists a reaction $R$ such that $F_{\vec{a}\rightarrow \vec{a}_i}(R) > F_{\vec{a}\rightarrow \vec{b}}(R)$.
Note that such a state $\vec{a_i}$ exists by the lemma's assumption.
Note that since $\vec{a}_{i-1} \rightarrow^1_{\vec{u}_i} \vec{a}_i$, then $\vec{a}_{i-1} \rightarrow^1_{\lambda\vec{u}_i} \vec{a'}$ for any $\lambda \in [0, 1]$.
In other words, every state along the line segment from $\vec{a_{i-1}}$ to $\vec{a_i}$ is reachable from $\vec{a}$.
Find the minimal $\lambda$ such that there exists a reaction $R'$ such that $R'$ is being applied on this line segment (formally, $\vec{u}_i(R') > 0$) and $F_{\vec{a}\rightarrow \vec{a'}}(R') = F_{\vec{a}\rightarrow \vec{b}}(R')$.
These minimal choices of $i$ and $\lambda$ ensure that for all $R'' \neq R'$, $F_{\vec{a}\rightarrow \vec{a'}}(R'') \leq F_{\vec{a}\rightarrow \vec{b}}(R'')$.

Let $S$ be an arbitrary reactant of $R'$, and let $r$ be the entry in the stoichiometry matrix corresponding to species $S$ and reaction $R'$.
Let $P_1,\dots,P_n$ be the reactions which produce species $S$, and let $p_i$ be the entries of the stoichiometry matrix corresponding to species $S$ and reactions $P_i$.
Note that by non-competition, $p1,\dots,p_n$ are nonnegative.
We can write the concentrations of
$S$ in $\vec{a'}$ and $\vec{b}$ as the initial concentration plus the amount changed by reaction application as follows:
\begin{align*}
    \vec{a'}(S) &= \vec{a}(S) + p_1F_{\vec{a} \rightarrow \vec{a'}}(P_1) + \dots +
    p_nF_{\vec{a} \rightarrow \vec{a'}}(P_n) +
    rF_{\vec{a} \rightarrow \vec{a'}}(R'),
    \\
    \vec{b}(S) &= \vec{a}(S) + p_1F_{\vec{a} \rightarrow \vec{b}}(P_1) + \dots
    + p_nF_{\vec{a} \rightarrow \vec{b}}(P_n) +
    rF_{\vec{a} \rightarrow \vec{b}}(R').
\end{align*}
Recall that $\vec{a'}$ was chosen such that $F_{\vec{a} \rightarrow \vec{a'}}(R') = F_{\vec{a} \rightarrow \vec{b}}(R')$ and for all reactions $R'' \neq R'$ (notably, the $P_i$ reactions), $F_{\vec{a} \rightarrow \vec{a'}}(R'') \leq F_{\vec{a} \rightarrow \vec{b}}(R'')$.
So we have $\vec{a'}(S) \leq \vec{b}(S)$.
Further, recall that $R'$ is applicable in $\vec{a'}$, so $\vec{a'}(S) > 0$, and so $\vec{b}(S) > 0$.
Since $S$ was arbitrary, all reactants needed to apply reaction $R'$ are available in $\vec{b}$, so $\vec{b}$ is not static.
\end{proof}

While Theorem~\ref{thm:main} is stated in Section~\ref{sec:CRNprogramming} in terms of mass-action kinetics, we reiterate that the theorem holds for any fair rate law (Definition~\ref{def:fair_rate_law}).
Previous work shows that mass-action is indeed a fair rate law:

\begin{lemma}\label{lem:mass-action_reachable_implies_segment_reachable}
Proven in~\cite{chen2014rate}: For any CRN, if $\vec{a}$ can reach $\vec{b}$ under mass action, then $\vec{a} \rightarrow \vec{b}$.
(This holds even if $\vec{b}$ takes infinite time to reach under mass action, i.e., it is the limit state.)
\end{lemma}

Towards proving the theorem, first, we must eliminate the possibility that although $\vec{a} \rightarrow \vec{b}$ in the nondeterministic kinetic model and $\vec{b}$ is a static state, that somehow the CRN may converge under the rate law to a dynamic equilibrium or to some oscillatory cycle of states, or that it does not converge at all.
These kinetic behaviors are associated with the following kinds of infinite paths in the nondeterministic kinetic model as described in Lemma~\ref{lem:nc_implies_no_equilibrium_implies_unbounded potential} below.

\begin{definition}
Given a CRN and a state $\vec{a}$, $\vec{a}$ has \emph{unbounded potential} if there exists a path $\vec{a} \rightarrow^1_{\vec{u_1}} \vec{a_1} \rightarrow^1_{\vec{u_2}} \vec{a_2} \rightarrow^1_{\vec{u_3}} \dots$ such that there exists a reaction $R$ such that $\sum_{i = 1}^{\infty}\vec{u_i}(R) = \infty$.
\end{definition}

\begin{lemma}\label{lem:nc_implies_no_equilibrium_implies_unbounded potential}
Assume a CRN is non-competitive.
If $\vec{a}$ does not converge to a static equilibrium under fair rate law kinetics, then $\vec{a}$ has unbounded potential.
\end{lemma}
\begin{proof}
There are two cases; either $\vec{a}$ converges to a dynamic equilibrium, or $\vec{a}$ does not converge.
If $\vec{a}$ converges to a dynamic equilibrium $\vec{c}$, then by the fair rate law assumption, $\vec{a} \rightarrow \vec{c}$.
Since $\vec{c}$ is a dynamic equilibrium, there exists a nonzero flux vector $\vec{u}$ such that $\vec{c} \rightarrow^1_{\vec{u}} \vec{c}$.
Consider the path $\vec{a} \rightarrow \vec{c} \rightarrow^1_{\vec{u}} \vec{c} \rightarrow^1_{\vec{u}} \dots$.
This path shows that $\vec{a}$ has unbounded potential.

Otherwise, $\vec{a}$ does not converge as $t \rightarrow \infty$.
In this case, intuitively, we use the assumption of non-convergence to construct a path with unbounded potential.
Formally, letting $\vec{s}_t$ be the state of the CRN starting at $\vec{a}$ under mass-action kinetics after time $t$, we will show how to find an infinite sequence of time points $t_0, t_1, \dots$ such that $\vec{a} \rightarrow \vec{s}_{t_0} \rightarrow \vec{s}_{t_1} \rightarrow \dots$ and this path has infinite flux on some reaction $R$, thus showing that $\vec{a}$ has unbounded potential.

Let $s(t)$ be the state reached at time $t$ starting from $\vec{a}$ under mass-action kinetics.
By negating the definition of convergence, non-convergence means that for any state $\vec{x} \in \mathbb{R}^n_{\geq 0}$, we can find an $\varepsilon \in \mathbb{R}$ such that for any time $t$, we can find a $t_0 > t$ such that there is a species $S$ such that $|s(t_0)(S) - \vec{x}(S)| \geq \varepsilon$, i.e., $s(t_0)$ is outside of the open ball of $\varepsilon$ radius centered at $\vec{x}$.
Let the initial state $\vec{a}$ be the $\vec{x}$ in the non-convergence definition, then let $\varepsilon_0 = \varepsilon$, take an arbitrary time $t$, and any $t_0 > t$.
Some species $S$ has $|s(t_0)(S) - \vec{x}(S)| \geq \varepsilon_0$, and by the fair rate law assumption, $\vec{a} \rightarrow s(t_0)$.
Then, similarly, letting $s(t_0)$ be the $\vec{x}$ in the non-convergence definition, let $\varepsilon_1 = \varepsilon$, an arbitrary $t > t_0$, and take any $t_1 > t$.
Now, some species $S$ has $|s(t_1)(S) - s(t_0)(S)| \geq \varepsilon_1$, and by the fair rate law assumption, $s(t_0) \rightarrow s(t_1)$.
Repeating this process yields an infinite path $\vec{a} \rightarrow s(t_0) \rightarrow s(t_1) \dots$ and an infinite sequence $\varepsilon_0, \varepsilon_1, \dots$ with the property that, given $i \in \mathbb{N}$, there is a  species $S$ such that $|s(t_i)(S) - s(t_{i-1})(S)| \geq \varepsilon_i$.
Note that we can choose each $t_i$ such that $\varepsilon_i \geq \varepsilon_{i-1}$.\footnote{To show this, towards contradiction assume the following proposition $\mathcal{P}$: for all choices of the infinite sequence of $t_i$, there is an infinite subsequence $t'_1 \dots$ of the $t_i$ such that $\varepsilon'_i < \varepsilon'_{i-1}$.
Choose an arbitrary infinite sequence of $t_i$; it must be that after some $t_j$, each $\varepsilon_k < \varepsilon_{k-1}$ for all $k > j$.
Otherwise, there would be an infinite subsequence of $t'_1 \dots$ of the $t_i$ with $\varepsilon'_i \geq \varepsilon'{i-1}$, contradicting proposition $\mathcal{P}$.
The sequence $t_k \dots$ show that the CRN converges, contradicting that the CRN does not converge.}
Since each $s(t_i)$ is at least $\varepsilon_1$ away from $s(t_{i+1})$, we have a path $\vec{a}\rightarrow s(t_0) \rightarrow s(t_1) \rightarrow \dots$ showing that $\vec{a}$ has unbounded flux.
\end{proof}

We prove that states which have a path to a static state have bounded potential, and so by the contrapositive of Lemma~\ref{lem:nc_implies_no_equilibrium_implies_unbounded potential} must converge to a static equilibrium under fair rate laws.

\begin{lemma}\label{lem:nc_implies_no_infinite_flux_path}
Assume a CRN is non-competitive.
If $\vec{a} \rightarrow \vec{b}$ and $\vec{b}$ is a static state, then $\vec{a}$ does not have unbounded potential.
\end{lemma}
\begin{proof}
Towards contradiction, assume $\vec{a}$ has unbounded potential.
Let $p_1$ be any path from $\vec{a} \rightarrow \vec{b}$.
Since $\vec{a}$ has unbounded potential, there is a path $p_2$ from $\vec{a}$ with some reaction $R$ which is applied with infinite flux.
Then that reaction $R$ is applied more in $p_2$ than in $p_1$ (since it must be applied with finite flux in the finite path $p_1$), so by Lemma~\ref{lem:nc_implies_two_paths_flux_difference_implies_one_is_not_static}, $\vec{b}$ is not static.
\end{proof}

All that remains is to prove that the static equilibrium reached by the fair rate law is in fact the same state $\vec{b}$ as assumed in the nondeterministic kinetic model.
First we prove that we cannot have two different static states $\vec{b}$ and $\vec{c}$ both reachable from $\vec{a}$.

\begin{lemma}\label{lem:nc_implies_no_two_static_states}
For non-competitive CRNs, if $\vec{a}\rightarrow \vec{b}$ and $\vec{a} \rightarrow \vec{c}$ and $\vec{b}$ and $\vec{c}$ are static states, then $\vec{b} = \vec{c}$.
\end{lemma}
\begin{proof}
Towards contradiction, assume $\vec{b} \neq \vec{c}$.
Then, without loss of generality, $\vec{a} \rightarrow \vec{c}$ applies some reaction $R$ more than $\vec{a} \rightarrow \vec{b}$.
So by Lemma~\ref{lem:nc_implies_two_paths_flux_difference_implies_one_is_not_static}, $\vec{b}$ cannot be static.
\end{proof}

Using this lemma, there is only one static state $\vec{b}$ reachable from $\vec{a}$.
The next lemma is a restricted version of Theorem~\ref{thm:main}, assuming that the starting state is $\vec{a}$.
After, we will show how the same lemma holds for any $\vec{a'}$ such that $\vec{a} \rightarrow \vec{a'}$.

\begin{lemma}\label{lem:nc_implies_static_state_is_static_equilibrium}
Assume a CRN is non-competitive.
If $\vec{a} \rightarrow \vec{b}$ and $\vec{b}$ is a static state, then $\vec{a}$ converges to $\vec{b}$ under any fair rate law.
\end{lemma}
\begin{proof}.
By Lemma~\ref{lem:nc_implies_no_infinite_flux_path}, $\vec{a}$ does not have unbounded potential.
So by the contrapositive of Lemma~\ref{lem:nc_implies_no_equilibrium_implies_unbounded potential}, $\vec{a}$ converges to a static equilibrium under mass action.
We will show that this static equilibrium must be $\vec{b}$.
Towards contradiction, assume $\vec{a}$ converges to some $\vec{c} \neq \vec{b}$ under mass action.
Then by the fair rate law assumption, $\vec{a} \rightarrow \vec{c}$.
Also note that $\vec{c}$ is a static state since it is a static equilibrium.
So Lemma~\ref{lem:nc_implies_no_two_static_states} implies $\vec{c} = \vec{b}$.
\end{proof}

Next we will show that the above holds for any state $\vec{a'}$ such that $\vec{a} \rightarrow \vec{a'}$.
This is done by showing that any reachable state $\vec{a'}$ can still reach the static state $\vec{b}$, and thus intuitively any reachable $\vec{a'}$ may replace $\vec{a}$ for all of the lemmas above.

\begin{lemma}\label{lem:nc_implies_stable_convergence}
For non-competitive CRNs, if $\vec{a} \rightarrow \vec{b}$ and $\vec{b}$ is a static state, then for all $\vec{a'}$ such that $\vec{a} \rightarrow \vec{a'}$, it must be that $\vec{a'} \rightarrow \vec{b}$.
\end{lemma}
\begin{proof}
There are two cases: given a fair rate law, $\vec{a'}$ either converges or does not converge to a static equilibrium.
If $\vec{a'}$ reaches a static equilibrium $\vec{c}$, then by Lemma~\ref{lem:mass-action_reachable_implies_segment_reachable}, $\vec{a'} \rightarrow \vec{c}$, so $\vec{a} \rightarrow \vec{c}$. 
Then Lemma~\ref{lem:nc_implies_no_two_static_states} implies $\vec{b} = \vec{c}$.
Otherwise, if $\vec{a'}$ does not reach a static equilibrium,
then Lemma~\ref{lem:nc_implies_no_equilibrium_implies_unbounded potential} implies $\vec{a}$ has unbounded potential.
However, since $\vec{a} \rightarrow \vec{b}$ and $\vec{b}$ is static, this contradicts Lemma~\ref{lem:nc_implies_no_infinite_flux_path}.
\end{proof}

Together, Lemmas~\ref{lem:nc_implies_static_state_is_static_equilibrium}~and~\ref{lem:nc_implies_stable_convergence} prove Theorem~\ref{thm:main} for any fair rate law.

\subsection{Proof of Optimization Procedure}\label{sec:optimization_proof}
Here we prove that the optimization procedure of Section~\ref{sec:optimization} does not change the state of convergence if the CRN is non-competitive.
For simplicity, we prove the theorem in the case that the optimization removes one reaction.
Removing many reactions is done by removing one reaction at a time.
If $\vec{a}$ is a vector of length $\Lambda$, then let $\vec{a}^{\setminus i}$ be the same vector without an entry for element $i$, i.e., the projection of $\vec{a}$ from the space $R_{\geq0}^\Lambda$ to the subspace $R_{\geq0}^{\Lambda \setminus i}$.
Intuitively, this maps states and flux vectors of a CRN to its optimized CRN (when just one species/reaction is removed).

\begin{theorem}
Assume a CRN is non-competitive, and consider its optimized CRN generated by removing a reaction $R$ with reactant $S$.
If $\vec{a} \rightarrow \vec{b}$ and $\vec{b}$ is a static state and $\vec{a}(S), \vec{b}(S) = 0$, then the optimized CRN has $\vec{a}^{\setminus S} \rightarrow \vec{b}^{\setminus S}$.
\end{theorem}
\begin{proof}
We write $\vec{a} \rightarrow^1_{\vec{u}_1} \vec{a}_1 \rightarrow^1_{\vec{u}_2} \dots \rightarrow^1_{u_k} \vec{b}$.
For the optimized CRN, we will show that the same sequence of flux vectors is a valid path for the optimized CRN which reaches the same state.
Formally, we will show $\vec{a}^{\setminus S} \rightarrow^1_{\vec{u}_1^{\setminus R}} \vec{a}'_1 \rightarrow^1_{\vec{u}_2^{\setminus R}} \dots \rightarrow^1_{u_k^{\setminus R}} \vec{b}^{\setminus S}$.

First note since $\vec{b} = \vec{M}\sum_{i=1}^k \vec{u}_i + \vec{a}$, that also $\vec{b}^{\setminus S} = \vec{M}'\sum_{i=1}^k (u_i^{\setminus R}) + \vec{a}^{\setminus S}$ where $\vec{M}'$ is the stoichiometry matrix for the optimized CRN.
This holds reactions producing $R$'s reactant now produce $R$'s products in $\vec{M}'$; and because $\vec{a}(S), \vec{b}(S) = 0$, any reactant of $R$ that is produced in the path from $\vec{a} \rightarrow \vec{b}$ must be consumed by reaction $R$ to produce the products in $\vec{b}$ (they must be consumed by $R$ due to non-competition).

Then it remains to show that $u_{i+1}^{\setminus R}$ is applicable at state $\vec{a}'_i$, noting that $\vec{a}'_i$ is not necessarily $\vec{a}_i^{\setminus S}$.
It helps to decompose the reactions into three sets: the removed reaction $\{R\}$, the set $\mathcal{T}$ of reactions which produced species $S$ in the original CRN, and the set of reactions $\mathcal{K}$ which did not produce $S$ in the original CRN so are unmodified by the optimization.
Consider an arbitrary species $A \neq S$; we will show that $\vec{a}'_i(A) \geq \vec{a}_i(A)$, implying that $\vec{u}^{\setminus{R}}_{i+1}$ is applicable in $\vec{a}'_i$ since it is applicable in $\vec{a}_i$.
We can determine the concentrations:

\begin{align}
\vec{a}'_i(A) = \vec{a}^{\setminus S}(A) &+ \sum_{K \in \mathcal{K}}\left( M'_{A,K} \sum_{j = 1}^i \vec{u}^{\setminus R}_{j}(K)\right) \\ 
&+ \sum_{T \in \mathcal{T}} \left( M'_{A,T} \sum_{j = 1}^i \vec{u}^{\setminus R}_{j}(T) \right)\\
\vec{a}_i(A) = \vec{a}(A) &+ \sum_{K \in \mathcal{K}}\left( M_{A,K} \sum_{j = 1}^i \vec{u}_j(K)\right) \\ 
&+ \sum_{T \in \mathcal{T}} \left( M_{A,T} \sum_{j = 1}^i \vec{u}_j(T) \right)\\
&+ M_{A,R} \sum_{j = 1}^i \vec{u}_j(R)
\end{align}

Note that $$\sum_{K \in \mathcal{K}}\left( M'_{A,K} \sum_{j = 1}^i \vec{u}^{\setminus R}_{j}(K)\right) = \sum_{K \in \mathcal{K}}\left( M_{A,K} \sum_{j = 1}^i \vec{u}_j(K)\right),$$ so it remains to show:
\begin{align}
\sum_{T \in \mathcal{T}} \left( M'_{A,T} \sum_{j = 1}^i \vec{u}^{\setminus R}_{j}(T) \right) \geq \\
\sum_{T \in \mathcal{T}} \left( M_{A,T} \sum_{j = 1}^i \vec{u}_j(T) \right)
+ M_{A,R} \sum_{j = 1}^i \vec{u}_j(R).
\end{align}

If $A$ is not produced in $R$, then $\vec{M}_{A,R} = 0$ and $\vec{M}'_{A,T} = \vec{M}_{A,T}$ so the terms are equal.
Otherwise, $A$ is produced in $R$.
Since reaction $R$ has only one reactant $S$ and the initial concentration of $S$ is zero, we know that the total flux through $R$ depends on the total flux through reactions in $\mathcal{T}$ (the reactions which produce $S$),

\begin{align}
\sum_{T \in \mathcal{T}}\left( M_{S,T} \sum_{j = 1}^i \vec{u}_j(T)\right) \geq \sum_{j = 1}^i \vec{u}_j(R),
\end{align}
\noindent

Due to the optimization procedure, the amount of $A$ produced by $T$ is equal to the original amount produced plus the amount produced by $R$ times the number of appearances of $S$ as a reactant, i.e., $\vec{M}'_{A,T} = \vec{M}_{A,T} + \vec{M}_{S,T}\vec{M}_{A,R}$, so:

\begin{align}
\sum_{T \in \mathcal{T}} \left( M'_{A,T} \sum_{j = 1}^i \vec{u}^{\setminus R}_{j}(T) \right) \\
= \sum_{T \in \mathcal{T}} \left( M_{A,T} \sum_{j = 1}^i \vec{u}^{\setminus R}_{j}(T) \right) &+ \vec{M}_{A,R} \sum_{T \in \mathcal{T}}  \left( \vec{M}_{S,T} \sum_{j = 1}^i \vec{u}^{\setminus R}_{j}(T) \right)\\
\geq \sum_{T \in \mathcal{T}} \left( M_{A,T} \sum_{j = 1}^i \vec{u}^{\setminus R}_{j}(T) \right) &+ \vec{M}_{A,R} \sum_{j = 1}^i \vec{u}_j(R).
\end{align}

Therefore $\vec{a}'_i(A) \geq \vec{a}_i(A)$ so the flux vector $\vec{u}^{\setminus R}_{i+1}$ is applicable at state $\vec{a}'_i$.
Since $i$ was arbitrary, we have constructed a path showing that $\vec{a}^{\setminus S} \rightarrow \vec{b}^{\setminus S}$.

\end{proof} %
\subsection{Non-competitive Bimolecular Rational Multiplication}\label{sec:BRMCRNs_proofs}
Here we show correctness for the construction from Figure~\ref{fig:rational_crn}.
We argue for any two numbers $p,q \in \mathbb{Z}$, our construction computes $y=\frac{p}{q}x$.

    First, we describe how to construct the CRN from Figure~\ref{fig:rational_crn}.
    Let $a.bc^\infty$ be the binary expansion of $\frac{p}{q}$ where $a \in \{0,1\}^i$, $b \in \{0,1\}^j$, and $c \in \{0,1\}^k$.
    Construct a CRN of the form given in Fig~\ref{fig:rational_crn}b with $n$ reactions ($n=i+j+k+3$) where each of the reactions (other than the first and last) is either of the form $L_r \rxn L_{r+1} + L_{r+1}$ or $R_r + R_r \rxn R_{r+1}$.
    Let us enumerate the bits in $a.bc^\infty$ (from left to right) as $b_i b_{i-1} \dots b_2 b_1 . b_{i+1} b_{i+2} \dots$
    For each bit $b_m$, where $0<m\leq n$, in $a.bc^\infty$, if $b_m = 1$ add the output species $Y$ as a product to reaction $m$.

    To prove correctness, it is sufficient to reason about the stoichiometry of one particular path to a static state (due to the non-competitive nature of this CRN).
    Given an ordering on species $(X,L_0,L_1,\dots,L_i,R_0,R_1,\dots,R_j,Y)$ and an ordering on reactions as listed in Fig~\ref{fig:rational_crn}b, consider $\vec{a} + \vec{M}\vec{v} = \vec{b}$ with initial state
    $\vec{a} = [x,0,\dots,0]$, 
    final state $\vec{b} = [0,0,\dots,\frac{p}{q}x]$,
    and a stoichiometry matrix as defined by the CRN:
    \begin{center}
      \includegraphics[scale=0.65]{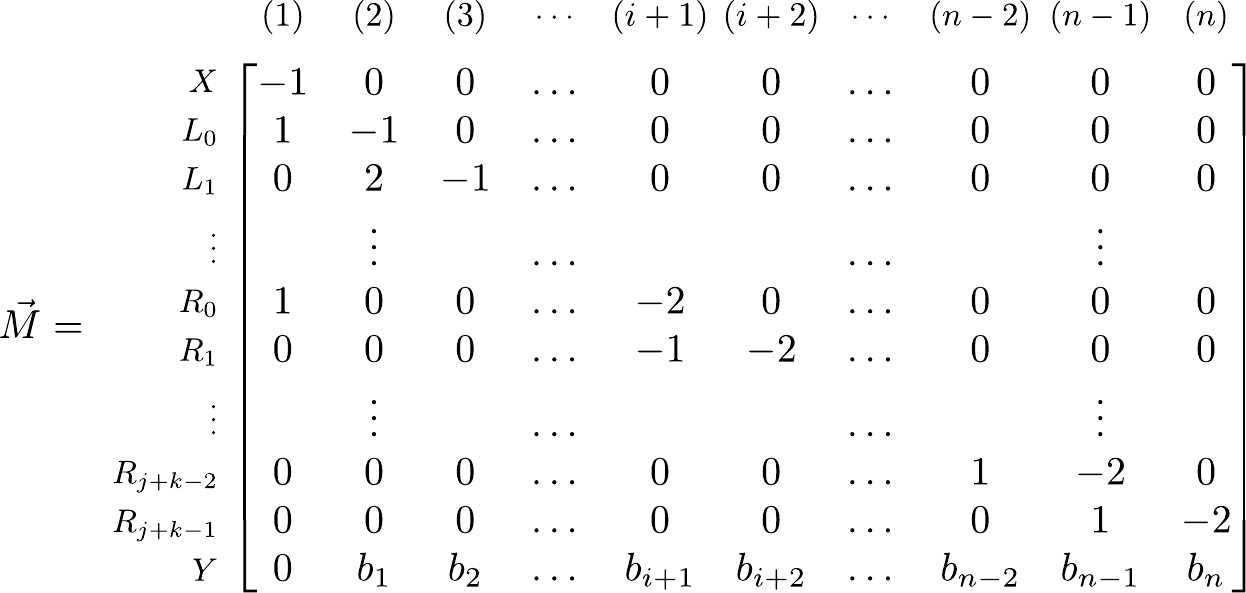}
    \end{center}

    We can solve for $\vec{v}$ to find
    $\vec{v} = [x,x,2x,4x,\dots,2^{i-1}x,\frac{x}{2},\frac{x}{4},\dots,\frac{x}{2^j},\dots,\frac{x}{2^{j+k}}]^T$.
    Our problem, however, is that $\vec{v}$ is not applicable at $a$.
    To remedy this, we decompose $v$ into $u_1, u_2, \dots, u_n$ such that each $u$ is applicable.
    \begin{align*}
      \vec{v} &= {\sum_{i=1}^n u_i} \text{ where } \\
      u_1 &= [\frac{1}{2}v[0],0,0,\dots,0]^T \\
      u_2 &= [0,\frac{1}{4}v[1],0,\dots,0]^T \\
      u_3 &= [0,0,\frac{1}{8}v[2],\dots,0]^T \\
      \vdots\\
      u_{i+1} &= [0,\dots \frac{1}{4}v[i],0 \dots, 0]^T \\
      u_{i+2} &= [0,\dots,0,\frac{1}{8}v[i+1] \dots, 0]^T \\
      \vdots\\
      u_{n-1} &= [0,\dots,\frac{1}{2^{j+k}}v[n-1]]^T \\
      u_n &= v - \sum_{i=1}^{n-1} u_i
    \end{align*}

    Then, $$\vec{a} + \vec{M}{\sum_{i=1}^{n}u_i} = \vec{b}.$$
    Thus, by Theorem~\ref{thm:main}, our construction stoichiometrically computes $y = \frac{p}{q}x$.

\subsection{Analogous Theorems for Stochastic Kinetic Models}\label{sec:stochastic_proofs}
Here we show that a theorem analogous to Theorem~\ref{thm:main} is also true for non-competitive CRNs in the stochastic model.
The stochastic model of CRNs differs from the concentration-, ODE-based kinetic models of CRNs mainly in that concentrations are replaced by discrete amounts of species and reaction applications are discrete events which change species' amounts by integer values.

We provide some basic definitions of reachability in the stochastic model.
It will be sufficient to reason only about reachability.\footnote{Typically stochastic CRNs are modeled as continuous time Markov processes, but our results hold as long as transition probabilities corresponding to applying a reaction are positive if all reactants for the reaction are positive.
In other words, the kinetics must obey a stochastic equivalent of the \emph{fair rate law} assumption in the continuous case.}
Note first that the stoichiometry matrix $\vec{M}$ is the same as the continuous model.
States of a CRN are an assignment of counts to each species, and so we can view them as vectors of nonnegative integers.
To define reachability by applying single reactions as discrete events, we say state $\vec{a} \rightarrow^1_{R} \vec{b}$ if there is a reaction $R$ such that $R$ is applicable in $\vec{a}$ and $\vec{b} = \vec{M}\vec{u}_R + \vec{a}$, where $\vec{u}_R(R') = 0$ for all $R' \neq R$ and $\vec{u}_R(R) = 1$.
Then we let $\rightarrow$ be the transitive reflexive closure of $\rightarrow^1$, i.e., reachability by applying zero or more reactions.
If $\vec{a} \rightarrow \vec{b}$, we can think of the existing sequence of $\rightarrow^1$ relations to get from $\vec{a}$ to $\vec{b}$ as a \emph{path}.

Note that the following lemma is analogous to Lemma~\ref{lem:nc_implies_two_paths_flux_difference_implies_one_is_not_static}, but has a simpler proof due to the discrete model.

\begin{lemma}\label{lem:stochastic_lemma}
Assume a CRN is non-competitive.
Consider two paths $p_1$ and $p_2$ leaving state $\vec{a}$. If $p_1$ has finite length and ends in state $\vec{b}$ and $p_2$ applies some reaction $R$ more than $p_1$, then $\vec{b}$ is not static.
\end{lemma}

\begin{proof}
The proof is mostly the same as Lemma~\ref{lem:nc_implies_two_paths_flux_difference_implies_one_is_not_static}.
Note that since reaction events are discrete, we can set $\vec{a'} := \vec{a}_i$ and set $R' := R$, while still ensuring that $F_{\vec{a} \rightarrow \vec{a'}}(R') = F_{\vec{a} \rightarrow \vec{b}}(R')$ and for all $R'' \neq R'$, $F_{\vec{a} \rightarrow \vec{a'}}(R'') \leq F_{\vec{a} \rightarrow \vec{b}}(R'')$.
The rest of the proof remains the same.
\end{proof}

Using the above lemma we can state a useful theorem which captures non-competitive CRN behavior in stochastic kinetic models.
Note that reactions as discrete events simplify the notion of a \emph{length} of a path as the number of reaction applications, or the number of $\rightarrow^1_{\vec{u}(R)}$ relations (excluding the ``empty'' reaction $\rightarrow^1_{[0,\dots,0]^T}$). 

\begin{theorem}
Assume a stochastic CRN is non-competitive.
If $\vec{a} \rightarrow \vec{b}$ via path $p$ with length $\ell_p$ and $\vec{b}$ is static, then
there is no path from $\vec{a}$ with length longer than $\ell_p$,
any path with length $\ell_p$ also ends in $\vec{b}$,
and any path with length shorter than $\ell_p$ ends in a state which is not static.
\end{theorem}
\begin{proof}
Let $p'$ be a path from $\vec{a}$ of length $\ell_{p'}$.

If $\ell_{p'} > \ell_{p}$, towards contradiction, $p'$ applies some reaction $R$ more than $p$, so by Lemma~\ref{lem:stochastic_lemma} $\vec{b}$ is not static which contradicts the lemma's assumption, so no such $p'$ exists.

If $\ell_{p'} \leq \ell_{p}$ and $p \neq p'$, then some reaction $R$ applies more in $p$ than in $p'$, so the state at the end of the path $p'$ cannot be static.
If $p = p'$, then both paths must end in $\vec{b}$.
\end{proof} 
\end{document}